%% file: neurips_2026.tex
\documentclass{article}
\usepackage{multirow}
\usepackage{capt-of} 

 \usepackage[preprint]{neurips_2026}

\usepackage[colorlinks=true, allcolors=blue]{hyperref}

\usepackage[utf8]{inputenc} 
\usepackage[T1]{fontenc}    
\usepackage{hyperref}       
\usepackage{url}            
\usepackage{booktabs}       
\usepackage{amsfonts}       
\usepackage{nicefrac}       
\usepackage{microtype}      
\usepackage{xcolor}         

\title{The Bayesian Learning Rule Beyond KL Geometry}

\author{%
  Sophia Sklaviadis\textsuperscript{$*$,1},
  ~Wu Lin\textsuperscript{$\ddagger$},
  ~\textbf{Andr\'{e} F. T. Martins}\textsuperscript{$*$}, \\
  ~\textbf{M\'{a}rio A. T. Figueiredo}\textsuperscript{$*$},
  ~\textbf{Thomas M\"{o}llenhoff}\textsuperscript{$\dagger$},
  ~\textbf{Mohammad Emtiyaz Khan}\textsuperscript{$\dagger$,\S}
  \\[1em]
  \textsuperscript{$*$}Instituto de Telecomunica\c{c}\~{o}es and
   \\ Instituto Superior T\'{e}cnico,
  Universidade de Lisboa, Portugal \\
  \textsuperscript{$\dagger$}RIKEN Center for Advanced Intelligence Project,
  Japan \\
  \textsuperscript{$\ddagger$}University of Central Florida, USA \\
  \textsuperscript{\S}TU Darmstadt and hessian.AI,  Germany
  \\[1em]
  \textsuperscript{1}\texttt{ssklaviadis@gmail.com}
}

\usepackage{wrapfig}

\usepackage{amsmath,amssymb, mathtools}
\usepackage{amsthm}
\usepackage{algorithm, algorithmic}

\newcommand{\algorithmicinit}{\textbf{Init:}}
\newcommand{\INIT}{\item[\algorithmicinit]}
\newcommand{\algorithmiccompute}{\textbf{Compute:}}
\newcommand{\COMPUTE}{\item[\algorithmiccompute]}

\usepackage{float}

\usepackage{xcolor}

\newcommand{\blue}[1]{\textcolor{TolMutedBlue}{#1}}

\usepackage{titletoc}

\usepackage[textsize=tiny]{todonotes}

\newcommand{\postdeformidx}{q}                
     
\newcommand{\varpost}{\pi_{\text{\vlambda}}}                   
\newcommand{\varcand}{\pi}                  
\newcommand{\prior}{\pi_0}           
\newcommand{\Qgauss}{\mathcal{N}_q}

\usepackage[capitalize,nameinlink]{cleveref}
\crefname{section}{Sec.}{Sec.}
\crefname{thm}{Thm.}{Theorem}
\crefname{appendix}{App.}{Appendices}
\crefname{algorithm}{Alg.}{Algorithms}
\crefname{equation}{Eq.}{Eqs.}
\crefname{figure}{Fig.}{Figs.}
\crefname{prop}{Prop.}{Props.}
\creflabelformat{equation}{#2\textup{#1}#3} 

\DeclareMathOperator*{\argmax}{arg \ max}

\definecolor{TolMutedBlue}{HTML}{332288}
\definecolor{TolMutedGreen}{HTML}{117733}
\definecolor{TolMutedPurple}{HTML}{AA4499}
\definecolor{TolMutedRed}{HTML}{BB3322}

\hypersetup{
    colorlinks=true,
    citecolor=TolMutedBlue,%
    linkcolor=TolMutedGreen,%
    urlcolor=TolMutedPurple,%
}

\input{defns}

\begin{document}

\maketitle

\begin{abstract}
The Bayesian learning rule (BLR) exploits the relation between Bayesian updating and the Kullback-Leibler divergence (KLD) to derive efficient additive updates for exponential family variational posteriors. The KL geometry is powerful but restrictive: combined with posterior families with parameter-dependent support it can lead to undefined objectives. 
We ask which non-KL geometries preserve the additive BLR computation.
Previous works that use  $f$-divergences suggest that this is only possible with the KLD. Here, we show that this is not true.
We show that additive information-geometric optimization algorithms can be derived for the class of mixed-space Bregman or Fenchel-Young discrepancies.
These discrepancies induce generalized maximum-entropy families, dual coordinates, and generalized-BLR canonical gradient updates, recovering the classical BLR. 
The generalized BLR leads to new algorithms that \textit{cannot} be derived from the KL-based BLR.
For the Tsallis negentropy with quadratic statistics, the generalized BLR yields
bounded-support and heavy-tailed $q$-Gaussian posteriors and an easy Newton-like implementation. Experiments up to CIFAR-100 illustrate controllable support and tail behavior and competitive performance.
\end{abstract}

\section{Introduction}

There is substantial work on generalizing Bayesian inference from the perspective of optimization: the variational posterior strikes a balance between information from the observations and adherence to the prior. Recently, \citet{knoblauch2022optimization} exploit the ``modularity'' of this optimization perspective, treating the variational family in which the posterior is sought, the loss connecting the posterior to the observations, and the divergence between the posterior and prior as independent design choices, ignoring formal interactions of these objects. In contrast, \citet{khan2023bayesian}'s Bayesian learning rule (BLR) is an iterated, conjugate inference extension of classical Bayes, preserving the standard couplings between the prior and entropic regularizer, and between prior and posterior families (\cref{sec:blr}). In the BLR, the updates stem from the interaction of Shannon entropy, KL geometry, and exponential-family posteriors, which is computationally powerful but restrictive: it ties the BLR to standard fixed-support exponential families. Posterior families with variable support or adaptive tail behavior fall outside the scope of the KL-based BLR. Here, we focus on a structural question raised by the BLR: are there generalizations of the KL-based variational objective that preserve the additive canonical-coordinate updates that free the BLR from its KL restriction?


Our contributions are: (1) We identify mixed-space Bregman or Fenchel-Young discrepancies, as defined by \citet{amari2010information, martins2022sparse} but only first connected to variational Bayes in \citet{fyvi}, as a class that preserves BLR-style additive canonical updates.
(2) We propose a generalized BLR update beyond the KL geometry and show that the classical BLR is recovered as a special case (\cref{sec:genblr}). (3) We instantiate the generalized BLR update for the Tsallis negative entropy, yielding parameter-dependent bounded-support and heavy-tailed $q$-Gaussian Newton-like variational learning methods (\cref{sec:qvon_summary}, App.~\ref{sec:qvon}).
Our new Newton-like optimizers require changing a few lines of code in state-of-the-art Improved Variational Online Newton (IVON)~\citep{shen2402variational}, a scalable Gaussian BLR-style optimizer for deep networks, to obtain heavy-tailed or bounded-support noise injection, maintain non-Gaussian posteriors, and target a generalized Bayesian objective using the maximum entropy (maxent) geometry associated with each objective.  
These optimizers enable generalized Bayesian learning beyond the KLD at scale (\cref{sec:experiments_summary}, App.s~\ref{app:experiments}, \ref{app:varying_q}).

By KL or mixed-space Bregman geometry we mean the convex-analytic structure defined by a reference score and convex function, which induce a mixed Bregman divergence: the score family and convex function determine the maxent posterior family whose convex conjugate determines dual coordinates, and a Hessian metric. Our terminology is motivated by the fact that a mixed Bregman divergence is the slack in a Fenchel-Young inequality.\footnote{We use \emph{Fenchel} rather than \emph{Legendre} deliberately. In convex analysis, the Legendre transform is the differentiable special case of the Fenchel conjugate \citep{boyd2004convex, touchette2005legendre}.}
Closest to our generalized BLR is \citet{alquier2021non}'s mirror-descent update that optimizes the sum of an expected loss and a $f$-divergence. 
Only when the $f$-divergence coincides with the KLD is it possible to separate the $f$-divergence ratio into a difference of linear expectations.
Uniquely, the KLD belongs both to the classes of $f$- and Bregman divergences \citep{amari2009alpha}.
The mixed-space Bregman divergence that we use recovers the KLD and BLR, but encompasses other divergences corresponding to generalized entropies. We optimize a generalized variational objective using dual coordinates of a relaxed class of reference scores in place of normalized priors, and generalized Fenchel posteriors in place of Gibbs posteriors. 
The resulting information-geometric algorithm has the same additive computation as the BLR, and admits a dual-space analogue of  \citet{bissiri2016general}'s sequential coherence, unlike \(f\)-divergence updates.
\begin{figure}
    \centering
    \includegraphics[width=\linewidth]{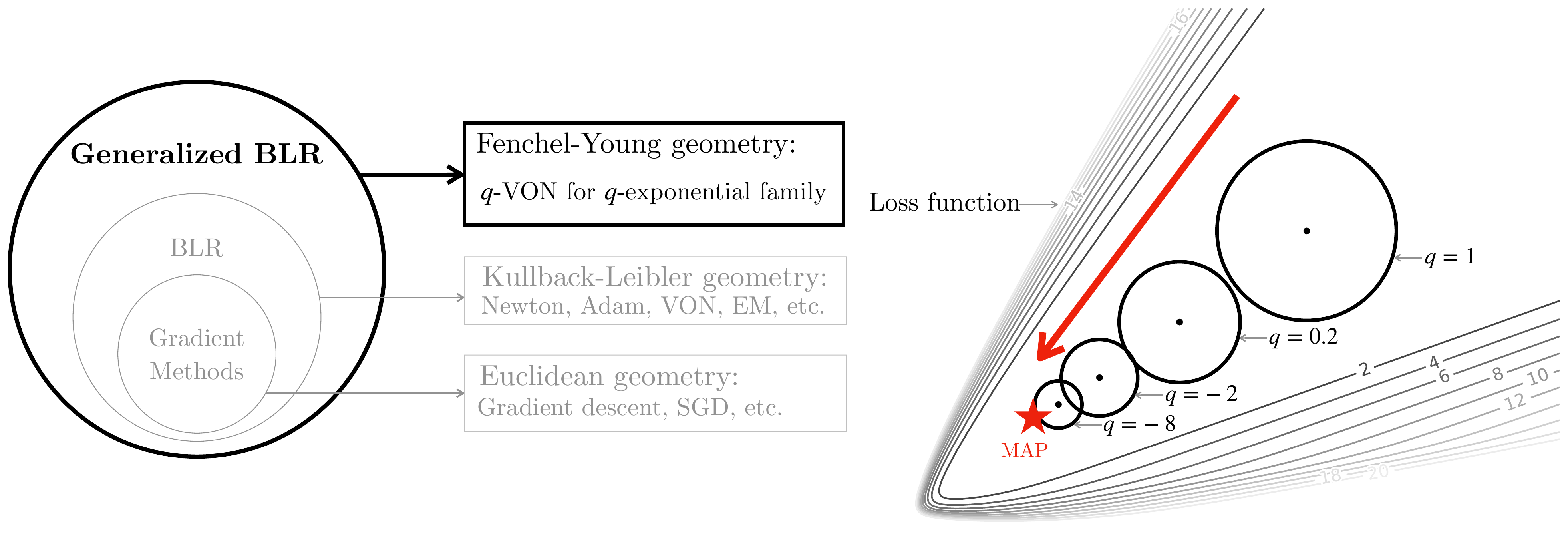}
    \caption{\textbf{Left:} We generalize the BLR by exploiting the geometry of Fenchel-Young discrepancies. We recover as special cases the BLR and the learning algorithms derived from it, and derive new algorithms using bounded-support and heavy-tailed $q$-exponential families.
    \textbf{Right:} Unlike Gaussians, bounded-support $q$-Gaussians allow control of the support by varying $q$. Thus, we can fit posteriors that closely track the MAP. The illustrated loss is that of 2D Bayesian logistic regression \citep{murphy2012machine}, each circle indicates the first standard deviation, and outer loss contours indicate higher loss values. The Gaussian ($q=1$) fit is farther from the MAP because its full support assigns non-negligible mass to high-loss regions; bounded-support $q$-Gaussians truncate this exploration.}
    \label{fig:fig}
\end{figure}

Parameter-dependent bounded-support $q$-Gaussian posteriors allow the posterior mean to closely track MAP solutions (\cref{fig:fig,fig:fig3}). By contrast, Gaussian variational means are far enough away from the MAP to avoid sampling parameter-space regions with high-loss.
The heavy-tailed $q$-Gaussian posterior updates we derive are reparametrizations of Student-$t$ distributions studied in robust Bayesian inference. Various aspects of robust Bayes can be reinterpreted in our framework, with the benefit of easy conjugate priors, but drawing these connections is beyond the scope of this paper; as is demonstrating state-of-the-art or scale. We show experiments up to CIFAR-100 (App.~\ref{app:experiments}), and defer large experiments, for example on quantization-aware training, pruning, and language model training with bounded variance to future work. We develop the theoretical framework and check that the proposed algorithms for variational training with controllable support and tail behaviors are competitive with their Gaussian IVON counterpart \citep{khan2018fast,osawa2019practical,shen2402variational}, without additional tuning.




\section{KLD Geometry and the BLR}\label{sec:blr}

The BLR relies on two levels of a common convex duality. At the density-space level, the Gibbs-Donsker-Varadhan variational formula is the Fenchel duality between the log-Laplace functional of a tilted prior, and the relative entropy with respect to that prior. Equivalently, if the prior is represented by a reference score, the KLD is the Fenchel-Young slack generated by the Shannon negentropy and that score. The density-space duality, restricted by finite moment constraints, results in the maximum entropy characterization of exponential families: with constant base measure, the exponential family \(\pi_{\text{\vlambda}}\) with canonical parameter $\vlambda$ maximizes the Shannon entropy among all densities with the same moments, and with a non-constant base measure it maximizes the relative entropy with respect to that base measure. At the parametric level, restricting the score to the finite-dimensional linear class
\(
\eta_{\text{\vlambda}}(\vtheta)=\langle\vlambda,\vT(\vtheta)\rangle
\)
induces the cumulant \(A(\vlambda)\). Under regularity and minimality, \(\vmu=\nabla A(\vlambda)\), and \((\vlambda,\vmu)\) are the dual coordinates on which the additive BLR update is based.

Fix a measurable space \((\Theta,\mathcal A)\) and a \(\sigma\)-finite reference measure \(h\). Let \(\mathcal P\) be a convex class of probability densities with respect to \(h\), and write \smash{\( \myexpect_\pi[f] = \int_\Theta f(\vtheta)\pi(\vtheta) d h(\vtheta)\)}.
For the Shannon negentropy \smash{\(\Omega_{\mathrm{Sh}}= -\mathcal H(\pi) = \int_\Theta \pi(\vtheta)\log \pi(\vtheta) d h(\vtheta) \)}, the Fenchel conjugate over densities 
\[
\Omega_{\mathrm{Sh}}^*(\eta) = \sup_{\pi\in\mathcal P} \left\{ \myexpect_\pi[\eta(\vtheta)]-\Omega_{\mathrm{Sh}}(\pi) \right\}
= \log\int_\Theta e^{\eta(\text{\vtheta})} d h(\vtheta),
\]
is log-normalization, whenever the integral is finite. The corresponding Fenchel-Young slack is
\[
\mathbb D_{-\mathcal H}(\eta,\pi) = \Omega_{\mathrm{Sh}}^*(\eta) + \Omega_{\mathrm{Sh}}(\pi) - \myexpect_\pi[\eta(\vtheta)]
\ge 0.
\]

The KLD can be derived as the Fenchel-Young slack generated by Shannon
negentropy and a reference score. In particular, if a prior density is
represented by score \(\eta_0\), as
\smash{\(\prior(\vtheta) = \exp\left( \eta_0(\vtheta)-\Omega_{\mathrm{Sh}}^*(\eta_0) \right)\)},
then \(\log\prior(\vtheta)=\eta_0(\vtheta)-\Omega_{\mathrm{Sh}}^*(\eta_0)\).
Hence, for any density \(\pi\),
\begin{align}
\mathbb D_{-\mathcal H}(\eta_0,\pi)
&= \Omega_{\mathrm{Sh}}^*(\eta_0) + \int_\Theta \pi(\vtheta)\log\pi(\vtheta) dh(\vtheta) - \myexpect_\pi[\eta_0(\vtheta)] \nonumber\\
&= \int_\Theta \pi(\vtheta) \left[ \log\pi(\vtheta) - \eta_0(\vtheta)
+ \Omega_{\mathrm{Sh}}^*(\eta_0) \right] dh(\vtheta) 
= \dkl{\pi}{\prior}.
\label{eq:kl_mixed_space}
\end{align}
Additive constants in \(\eta_0\) do not matter, since they change \(\Omega_{\mathrm{Sh}}^*(\eta_0)\) by the same constant.

The Gibbs or Donsker-Varadhan variational formula follows by applying the same slack to the tilted score \(\eta_0-\loss\). The normalizer of the tilted prior is
\[
\mathcal Z = \int_\Theta e^{-\loss(\text{\vtheta})}\prior(\vtheta) d h(\vtheta)
= \exp\left(\Omega_{\mathrm{Sh}}^*(\eta_0-\loss) - \Omega_{\mathrm{Sh}}^*(\eta_0) \right).
\]
Therefore, for every \(\varcand\in\mathcal P\) with \(\varcand\ll\prior\),
\begin{align}
\myexpect_\varcand[\loss(\vtheta)] + \dkl{\varcand}{\prior}
&= \myexpect_\varcand[\loss(\vtheta)] + \Omega_{\mathrm{Sh}}(\varcand) - \myexpect_\varcand[\eta_0(\vtheta)] + \Omega_{\mathrm{Sh}}^*(\eta_0) \nonumber\\
&= -\log\mathcal Z + \Omega_{\mathrm{Sh}}^*(\eta_0-\loss) + \Omega_{\mathrm{Sh}}(\varcand) - \myexpect_\varcand[\eta_0(\vtheta)-\loss(\vtheta)] \nonumber\\
&= -\log\mathcal Z + \mathbb D_{-\mathcal H}(\eta_0-\loss,\varcand)
\ge -\log\mathcal Z .
\label{eq:dv_from_fy}
\end{align}
Taking the infimum over candidates \(\varcand\) we get the Gibbs-Donsker-Varadhan variational formula
\begin{equation} \label{eq:dv_formula}
    - \log \mathcal{Z} = \inf_{\varcand\in\mathcal P: \varcand\ll\prior} \myexpect_\varcand[\loss(\vtheta)] + \dkl{\varcand}{\prior}.
\end{equation}
Conversely, sharpness of the Gibbs bound uniquely characterizes the KL regularizer (App.~\ref{app:kld}).
Equality in \cref{eq:dv_formula} holds when the Fenchel-Young slack in \cref{eq:dv_from_fy} vanishes, that is, when \(\varcand\) is the normalized density induced by the tilted score \(\eta_0-\loss\), namely the Gibbs posterior:
\[
\varcand_*(\vtheta) = \exp\left( \eta_0(\vtheta)-\loss(\vtheta) - \Omega_{\mathrm{Sh}}^*(\eta_0-\loss) \right) = \frac{e^{-\loss(\text{\vtheta})}\prior(\vtheta)}{\mathcal Z}.
\]

This duality appears in several literatures. In statistical mechanics and information theory, free-energy minimization and maximum Shannon entropy subject to moment constraints yield exponential families \citep{gibbs1902elementary,jaynes1957information,williams1980bayesian,Zellner1988}.
In generalized Bayesian inference, the negative log-likelihood is replaced by a general loss while retaining the KLD penalty to the prior \citep{zhu2014bayesian,bissiri2016general,ghosh2016robust,jewson2018principles,miller2019robust}.
In PAC-Bayesian learning theory, the KLD measures the complexity of moving from a fixed prior to a data-dependent posterior \citep{zhang1999theoretical,catoni2007pac,grunwald2020fast}. 
In variational inference, the ELBO is equivalently a KLD-regularized expected-fit objective or a KL projection of the exact posterior onto a tractable family \citep{wainwright2008graphical,hoffman2013stochastic,Blei2017}.
Instead of replacing the KLD in \cref{eq:dv_formula}, we use the Fenchel-Young structure that generates the KLD as the template for the mixed-space Bregman construction in \cref{sec:genblr}.\footnote{ A different line of thought in generalized VI is optimizes over a variational family by minimizing alternative divergences to the Bayesian posterior \citep{minka2005divergence, hernandez2016black, li2016renyi}. }

\subsection{Exponential Family Posteriors from Moment-Constrained Maxent}\label{ssec:expfam}
When the posterior belongs to a regular, minimal exponential family (expfam), through a constant base measure determined by finite moment constraints, the duality between log-normalization and negative entropy characterizes the posterior as the unique Shannon maxent distribution. 
To see this, first we fix a base measure $b(\vtheta) d\vtheta$ and a finite sufficient statistic
\begin{align}\label{eq:suff_stat} 
\vT(\vtheta) = \left(T_1(\vtheta), ..., T_P(\vtheta) \right).
\end{align}
For each \smash{$\vlambda\in\real^P$}, the linear score \smash{\(\eta_{\text{\vlambda}}(\vtheta) = \myang{\vlambda,\vT(\vtheta)} \)}
defines an exponential tilt \smash{\(e^{\eta_{\text{\vlambda}}(\text{\vtheta})} b(\vtheta) d\vtheta \)}, which is not, in general, a probability measure. For \smash{$\vlambda \in \Lambda = \{\vlambda \in \real^P: \int b(\vtheta)\exp \left(\eta_{\text{\vlambda}}(\vtheta)\right) d\vtheta < \infty \}$}, the unique additive correction is the cumulant or log-partition function 
\smash{\(
A(\vlambda) = \log \int b(\vtheta)\exp \left( \langle \vlambda, \vT(\vtheta) \rangle \right) d\vtheta
 \)},
defined on $\Lambda$, and the density induced by $h$ and $\vT$ has the familiar exponential family form
\begin{align}
\varpost(\vtheta) = b(\vtheta)\exp \left( \langle \vlambda, \vT(\vtheta) \rangle - A(\vlambda)
\right).
\label{eq:expfam}
\end{align}
Note that $\vlambda$ is introduced as the coefficient of the finite-dimensional score, and $A(\vlambda)$ is derived as the log-normalizer of the corresponding exponential tilt, which is the measure-theoretic operation that uses the score $\eta_{\text{\vlambda}}$ to tilt the base measure, thereby characterizing the expfam form.\footnote{The family is called regular if $\Lambda$ is a nonempty open subset of $\real^P$.
It is called minimal if there do not exist $(a, \vb) \in \real \times \real^P$, $(a,\vb) \neq (0, \vzero)$, such that \(a + \myang{\vb, \vT(\vtheta)} = 0 \) for almost every $\vtheta$. Under minimality, the natural parameter is identifiable, that is \(\pi_{\text{\vlambda}_1} = \pi_{\text{\vlambda}_2} \text{ a.e. } \Rightarrow \vlambda_1 = \vlambda_2 \), and $A$ is strictly convex.} 

Differentiating \smash{\( \int \varpost (\vtheta) d\vtheta = 1 \)} with respect to $\vlambda$ we obtain the moment map
\smash{\(
\nabla_{\text{\vlambda}} A(\vlambda) = \myexpect_{\varpost} [\vT(\vtheta)] = \vmu(\vlambda)
\)}.
Thus, natural and expectation parameter form a Legendre dual coordinate system. 
For simplicity, we assume the base measure is constant. In this case, the member of the expfam with moment vector \( \vmu(\vlambda) \) is the unique Shannon entropy maximizer among all densities $\pi \in \mathcal P$, that satisfy the same moment constraints (see App.~\ref{ssec:shannon_maxent} for a proof).\footnote{More generally, if the base measure is not constant, a corresponding statement can be made in terms of maximum relative Shannon entropy with respect to that base measure.}

\subsection{The BLR in Exponential Family Coordinates}\label{ssec:blr}
\citet{khan2023bayesian} apply \Cref{eq:dv_formula} to unify various learning algorithms under the Bayesian Learning Rule (BLR).
The BLR is obtained by restricting the DV optimization in \cref{eq:dv_formula} to the set of regular, minimal exponential families \smash{$\mathcal Q \subseteq \mathcal P$}, and exploiting the \smash{$(\vlambda, \vmu)$} dual pair. Assuming, in addition, that the prior belongs to the same family, \( \prior(\vtheta) = b(\vtheta) \exp \left( \langle \vlambda_0, \vT(\vtheta) \rangle - A(\vlambda_0) \right) \), and taking $b$ to be constant for simplicity, \smash{$\mathbb{D}_{-\mathcal{H}} (\log \prior, \varpost) = A(\vlambda_0) + A^*(\vmu(\vlambda) ) - \langle \vlambda_0, \vmu \rangle$}, where $A^*$ is the Legendre dual of $A$, and coincides with the negative Shannon entropy \smash{$ A^*(\vmu(\vlambda)) = - \mathcal{H}(\varpost)$}.
Differentiating $\mathbb{D}_{-\mathcal{H}} (\log \prior, \varpost)$ with respect to the expectation parameter, 
\smash{\( 
\nabla_{\text{\vmu}} \dkl{\varpost}{\prior} = \vlambda - \vlambda_0
\)}.

For exponential families, the Fisher information matrix $\vF(\vlambda)$ can be shown to equal \( \nabla^2_{\text{\vlambda}}A(\vlambda) \), which allows us to use the chain rule to express natural gradients $\widetilde{\nabla}$ as reparametrized regular gradients, \smash{$\widetilde{\nabla}_{\text{\vlambda} } = \vF(\vlambda)^{-1} \nabla_{\text{\vlambda} } = \nabla_{\text{\vmu} }$}. The expfam restriction renders the natural gradient of the KLD affine in $\vlambda$, and the natural-gradient step into a mirror-descent step with an additive $\vlambda$ update
\begin{align}
\vlambda & \gets \vlambda - \rho \nabla_{\text{\vmu}} \left( \myexpect_{\varpost}[\ell(\vtheta)] + \dkl{\varpost}{\prior} \right) 
= (1 - \rho) \vlambda - \rho \left( \nabla_{\text{\vmu}} \myexpect_{\varpost}[\ell(\vtheta)] - \vlambda_0 \right),
\label{eq:blr_additive}
\end{align}
where $\rho$ is a step size or learning rate. 

Under conjugacy, we can simplify further. There exists a natural parameter $\vlambda'$ such that
\begin{align}
\pi_{\text{\vlambda}'}(\vtheta) = \frac{1}{\mathcal Z} \exp(-\loss(\vtheta)) \prior(\vtheta) 
\Rightarrow \loss(\vtheta) &= \log \prior (\vtheta) - \log \pi_{\text{\vlambda}'}(\vtheta) - \log \mathcal Z \\
&= \langle \vlambda_0 - \vlambda', \vT(\vtheta) \rangle + A(\vlambda') - A(\vlambda_0) - \log \mathcal Z 
\end{align}
Thus, \smash{\( \myexpect_{\varpost} \loss(\vtheta) = \langle \vlambda_0 - \vlambda', \vmu (\vlambda) \rangle + \text{const.}\)}, and \(\nabla_{\text{\vmu}}\myexpect_{\pi_{\text{\vlambda}}}[\loss(\vtheta)] = \vlambda_0 - \vlambda'\), so \cref{eq:blr_additive} simplifies to
\footnote{Writing \Cref{eq:blr_additive} in terms of densities by substituting the update for $\vlambda$ into the exponential family form of $\varpost$ results in a multiplicative BLR update, which under conjugacy simplifies to 
\begin{align}  \label{eq:blr_measure}
\pi_{\text{\vlambda}}(\vtheta) 
\gets \pi_{\text{\vlambda}}(\vtheta)^{1-\rho} \pi_{\text{\vlambda}'}(\vtheta)^{\rho}
= \frac{1}{\mathcal{Z}_\rho } \pi_{\text{\vlambda}}(\vtheta)^{1-\rho} \left[\exp(-\loss(\vtheta))\prior(\vtheta) \right]^{\rho}.
\end{align}
}
\[
\vlambda \gets (1-\rho)\vlambda + \rho \vlambda'.
\]
For conjugate expfam models, a single BLR step coincides with classical Bayes' rule. More generally, outside conjugacy, the BLR optimizes the Bayesian objective whose minimizer over $\mathcal{P}$ is the exact posterior. Many algorithms from optimization, statistics, and deep learning arise as special cases by varying the form of the exponential family. New algorithms have also been inspired by the general form of the BLR, including the Variational Online Newton algorithm that we extend in \cref{sec:qvon_summary} and App.~\ref{sec:qvon} to exemplify our proposed generalized BLR.

\subsection{Limitations of the KLD and the BLR}

KLD geometry has certain limitations. The KLD requires absolute continuity: if $\varpost$ places mass outside the support of $\prior$, then $\dkl{\varpost}{\prior}=+\infty$. Thus posterior families with bounded, parameter-dependent support, and more generally families that cannot be represented as ordinary exponential families with fixed support, fall outside the scope of the KL-based BLR. Figs.~\ref{fig:fig} and~\ref{fig:fig3} illustrate why bounded-support families are useful: they can avoid sampling high-loss regions and allow the variational mean to approach MAP solutions more closely than Gaussian approximations.

How then can we generalize the KLD so as to preserve the BLR's efficient, additive updating? 
In \cref{sec:genblr} we show that Fenchel-Young discrepancies mediate exactly such a generalization. They yield generalized maxent families with well-defined dual parameters and generalized canonical gradients, allowing BLR-style updates under bounded, parameter-dependent support and beyond. 
Within the class of updates obtained by minimizing the sum of expected cumulative loss and an \(f\)-divergence, sequential coherence as defined by \citet{bissiri2016general} singles out the Kullback-Leibler divergence. In that sense, an \(f\)-divergence generalization of the KLD, like the \(f\)-divergence-regularized updates of \citet{alquier2021non}, cannot be used to define Bayes-like additive updates. 
Our mixed-space Bregman generalization of the KLD avoids this problem. The generalized BLR respects a dual-space analogue of sequential coherence: losses additively tilt the reference score, so processing losses sequentially or jointly yields the same tilted score. 
An interesting question that we hope to address in the future is whether mixed-space Bregman divergences uniquely recover BLR-style updates when we also wish to preserve the linear expectation of the loss term in \cref{eq:dv_formula}.

\section{Fenchel-Young BLR}\label{sec:genblr}

To derive a generalized BLR beyond KL geometry, first we introduce mixed-space Bregman divergences to which we refer interchangeably as Fenchel-Young discrepancies (FYD). We call \(\mathbb D_\Omega(\eta,\pi)\) a \emph{mixed-space} Bregman because its two arguments belong to different spaces: \smash{\(\eta\)} is a dual score, while \smash{\(\pi\)} is a density. This generalization of the KLD forms the basis of our proposal and allows our construction to remain meaningful for generalized entropies that may be nonsmooth, extended-valued, or have parameter-dependent support. We show how FYDs lead to generalized maxent families, for which we can define dual coordinates. Finally, we derive a generalized variational objective and learning rule. FYDs are also studied as discriminative losses by \citet{blondel2019learning, martins2022sparse}, and their information-geometric properties are discussed in \citet{amari2010information, nielsen2022statistical}. Here, we use them in a novel way to generalize the Bayesian learning rule optimizing the generalized objective defined in \citet{fyvi}.

\subsection{Generalizing the KLD to the Fenchel-Young Discrepancy} \label{ssec:fyd}

Fix a proper convex function $\Omega: \mathcal P \rightarrow  \real \cup \{ \infty \}$ on densities. The Fenchel conjugate of \(\Omega\) 
\[
\Omega^*(\eta) = \sup_{\pi \in\mathcal P} ~ \myexpect_\pi[\eta(\vtheta)] - \Omega(\pi)
\]
acts on measurable scores \smash{\(\eta:\Theta\to\mathbb R\cup\{\pm\infty\}\)}, for which the pairing
\smash{\( \langle \eta,\pi\rangle = \myexpect_\pi[\eta(\vtheta)] = \int_\Theta \eta(\vtheta)\pi(\vtheta) dh(\vtheta)\)} is well defined. 
For any \(\eta\in\operatorname{dom}\Omega^*\) and \(\pi\in\operatorname{dom}\Omega\) with finite pairing, the Fenchel-Young discrepancy is \citep{martins2022sparse, amari2010information}
\begin{align}\label{eq:fyd}
\mathbb{D}_\Omega(\eta,\pi) = \Omega^*(\eta) + \Omega(\pi) - \myexpect_\pi[\eta(\vtheta)].
\end{align}
The first argument is a reference score, not necessarily the log-density of a normalized prior. 
Since every \(\pi\in\mathcal P\) integrates to one,
\( 
\Omega^*(\eta+c) = \sup_{\pi\in\mathcal P} \left\{ \myexpect_\pi[\eta(\vtheta)]+c-\Omega(\pi) \right\}
= \Omega^*(\eta)+c,
\)
and therefore \(\mathbb D_\Omega(\eta+c,\pi)=\mathbb D_\Omega(\eta,\pi)\).
For general \(\Omega\), when the supremum defining \(\Omega^*(\eta)\) is attained, any maximizer
\( 
\pi_\eta \in \argmax_{\pi\in\mathcal P} \left\{ \myexpect_\pi[\eta(\vtheta)]-\Omega(\pi) \right\}
\)
is a density induced by the score \(\eta\). Equivalently, \(\pi_\eta\in\partial\Omega^*(\eta)\) with respect to the
\((\eta,\pi)\) pairing. 

\subsection{Generalizing the Exponential Family}
\label{ssec:maxent}

Dual coordinates related to the FYD can be defined for a much larger class of distributions that the expfam, including deformed exponential families \citep{amari2016information, naudts2002deformed}. Our generalization preserves the same logic we followed in \cref{ssec:expfam} to derive exponential families and characterize their dual coordinates. We begin with a finite-dimensional $\eta_{\text{\vlambda}}$, thereby identify a convex potential $\Psi^*(\vlambda) = \Omega^*(\eta_{\text{\vlambda}})$, and only when the supremum defining $\Omega^*(\eta_{\text{\vlambda}})$ is attained by a density, do we recover the corresponding member $\pi_{\text{\vlambda}}$ of the induced maximum-entropy family. The generalized Fisher we use is distinct from both the invariant and conformally flat geometries of \citet{amari2012geometry}, and it is a generalization of the metric associated with the derivation of deformed exponential families subject to linear (rather than escort) expectation constraints discussed in \citet{naudts2011generalised}.

Starting with a fixed measurable statistic $\vT(\vtheta)$ as in \cref{eq:suff_stat}, define the moment vector $\vmu = \myexpect_{\pi}[\vT(\vtheta)]$, and associate with it the Lagrange multiplier $\vlambda \in \mathbb{R}^P$. The corresponding finite-dimensional score family is $\eta_{\text{\vlambda}}(\vtheta)=\langle\vlambda,\vT(\vtheta)\rangle
$, defined up to an additive constant.\footnote{An additional constant statistic $T_0 = 1$ in $\vT$ can absorb normalization constants into $\vlambda$.} We define $\Psi$ to assign to each moment vector $\vmu$ the optimal value of the corresponding constrained minimization problem
\begin{align}\label{eq:maxent_psi}
    \Psi(\vmu) =
    \inf_{\pi\in\mathcal P} \Omega(\pi)
     \text{ s.t. } \underset{\pi}{\myexpect} [\vT(\vtheta)] = \vmu, \text{ with domain } \mathcal M = \left\{\underset{\pi}{\myexpect}[\vT(\vtheta)]: \pi \in \mathcal P, \Omega(\pi) < \infty \right\}.
\end{align}
We do not assume that the infimum in \Cref{eq:maxent_psi} is attained for every $\vmu$. When $\Omega=-\mathcal H$, \Cref{eq:maxent_psi} reduces to the classical Shannon maxent, and under the constant base measure convention of \cref{ssec:fyd}, whenever the minimizer exists it is the expfam in \cref{eq:expfam} with natural parameter $\vlambda = \nabla A^*(\vmu)$. In that case, $\Psi(\vmu) = A^*(\mu)$ and $\Psi^*(\vlambda) = A(\vlambda)$. 

\begin{prop}\label{prop:psi_star}
For every $\vlambda \in \real^P$,
\( 
\Psi^*(\vlambda) = \Omega^*(\eta_{\text{\vlambda}}).
\)
\end{prop}

\begin{proof}
By definition of the convex conjugate,
\begin{flalign*}
\Psi^*(\vlambda) &= \sup_{\text{\vmu} \in \mathcal M} \left\{ \langle \vlambda, \vmu \rangle - \Psi(\vmu) \right\} 
= \sup_{\text{\vmu} \in \mathcal M} \ \  \sup_{\pi \in \mathcal P: \myexpect_{\pi}[\text{\vT}] = \text{\vmu}} \left\{ \langle \vlambda, \vmu \rangle - \Omega(\pi) \right\} \\
&= \sup_{\pi \in \mathcal P} \left\{ \left \langle \vlambda, \underset{\pi}{\myexpect} [\vT(\vtheta)] \right \rangle - \Omega(\pi) \right\} 
= \sup_{\pi \in \mathcal P} \left\{ \underset{\pi}{\myexpect} [\eta_{\text{\vlambda}}(\vtheta)] - \Omega(\pi) \right\}
= \Omega^*(\eta_{\text{\vlambda}}). \qedhere
\end{flalign*}
\end{proof}
\cref{prop:psi_star} identifies the finite-dimensional convex potential associated with the score family. However, it does not yet identify a family of densities, it only shows how the infinite dimensional moment problem projects to canonical coordinates.
Next, we determine when a density obtained from a finite dimensional score solves the infinite dimensional problem in \cref{eq:maxent_psi}.
\begin{prop} \label{prop:score_optimality}
Assume that for $\vlambda \in \real^P$ the supremum defining $\Omega^*(\eta_{\text{\vlambda}})$ is attained by $\pi_{\text{\vlambda}} \in \mathcal P$:
\(
\Omega^*(\eta_{\text{\vlambda}}) = \myexpect_{\pi_{\text{\vlambda}}}[\eta_{\text{\vlambda}}(\vtheta)] - \Omega(\pi_{\text{\vlambda}})
\).
Define
\(
\vmu(\vlambda) = \myexpect_{\pi_{\text{\vlambda}}}[\vT(\vtheta)]
\).
Then $\pi_{\text{\vlambda}}$ minimizes \cref{eq:maxent_psi}. 
\end{prop}

\begin{proof}
By the optimality of $\pi_{\text{\vlambda}}$ in the definition of $\Omega^*(\eta_{\text{\vlambda}})$, for every \smash{$\pi \in \mathcal P$},
\smash{\(
\Omega(\pi) - \myexpect_{\pi}[\eta_{\text{\vlambda}}(\vtheta)]
\ge \Omega(\pi_{\text{\vlambda}}) - \myexpect_{\pi_{\text{\vlambda}}}[\eta_{\text{\vlambda}}(\vtheta)]
\)}.
Let $\pi$ satisfy the same moment constraints as $\pi_{\text{\vlambda}}$, namely $\myexpect_{\pi}[\vT (\vtheta)] = \vmu(\vlambda)$. Then
\begin{align}
\myexpect_{\pi}[\eta_{\text{\vlambda}}(\vtheta)]
&= \langle \vlambda, \myexpect_{\pi}[\vT(\vtheta)] \rangle 
= \langle \vlambda, \myexpect_{\pi_{\text{\vlambda}}}[\vT(\vtheta)] \rangle 
= \myexpect_{\pi_{\text{\vlambda}}}[\eta_{\text{\vlambda}}(\vtheta)]. 
\end{align}
Hence $\Omega(\pi) \ge \Omega(\pi_{\text{\vlambda}})$, which proves optimality. If $\Omega$ is strictly convex, this minimizer is unique.
\end{proof}
In \cref{prop:score_optimality} the infinite dimensional optimizer is certified by a finite dimensional score lying in the span of the constraint functions. This motivates the induced family
\[
\mathcal{Q}_{\Omega, \text{\vT}} = \left\{ \pi_{\text{\vlambda}} \in \mathcal P : \pi_{\text{\vlambda}} \text{ attains } \Omega^*(\eta_{\text{\vlambda}}), \vlambda \in \mathbb{R}^P \right\}.
\]
\cref{prop:psi_star,prop:score_optimality} are the only occasions where the full density-space duality is needed.
Once the induced family $Q_{\Omega,\text{\vT}}$ and the potential $\Psi^*$ are identified, all subsequent geometry and updates are finite-dimensional in the dual coordinates $(\vlambda,\vmu)$.
From here on, $\pi_\lambda$ denotes a member of $\mathcal{Q}_{\Omega, \text{\vT}}$. Whenever the maximizer is unique and the conjugate is differentiable at $\eta_{\text{\vlambda}}$, we write
\(
\pi_{\text{\vlambda}} = \nabla \Omega^*(\eta_{\text{\vlambda}})
\).
When $\Psi^*$ is differentiable at $\vlambda$ and the corresponding optimizer $\pi_{\text{\vlambda}}$ has finite moments,
\[
\vmu(\vlambda) = \nabla_{\text{\vlambda}} \Psi^*(\vlambda)
= \mathbb{E}_{\pi_{\text{\vlambda}}}[\vT(\vtheta)].
\]
We further write \( \vlambda(\vmu) = \nabla_{\text{\vmu}} \Psi(\vmu)\).
Thus $\vlambda$ and $\vmu$ play the roles of canonical and expectation parameters, as in exponential families, but for the larger class of maxent families induced by $\Omega$.

Finally, whenever $\pi_{\text{\vlambda}} \in \mathcal{Q}_{\Omega, \text{\vT}}$ solves \cref{eq:maxent_psi} at $\vmu$, the FYD admits the mixed-space form
\[
\mathbb{D}_\Omega(\eta_{{\text{\vlambda}}_0},\pi_{\text{\vlambda}}) =
\Psi^*(\vlambda_0) + \Psi(\vmu) - \langle \vlambda_0, \vmu \rangle.
\]
Indeed, \cref{prop:psi_star} gives $\Omega^*(\eta_{{\text{\vlambda}}_0}) = \Psi^*(\vlambda_0)$, optimality of $\pi_{\text{\vlambda}}$ in \cref{eq:maxent_psi} gives $\Omega(\pi_{\text{\vlambda}}) = \Psi(\vmu)$, and \smash{$\mathbb{E}_{\pi_{\text{\vlambda}}}[\eta_{{\text{\vlambda}}_0}(\vtheta)] = \langle \vlambda_0, \vmu \rangle$}. This is directly analogous to the mixed-space Bregman form of the KLD in \cref{eq:kl_mixed_space} and enables BLR-style additive updates in canonical coordinates. Importantly, the first argument is now a reference score $\eta_{{\text{\vlambda}}_0}$, not necessarily a normalized density. When $\Omega = -\mathcal H$ we recover the classical exponential family. When $\Omega = \Omega_{2-q}$, the induced family is the deformed $q$-exponential (proven in App.~\ref{ssec:tsallis_maxent}). In App.~\ref{ssec:qvon}, we show the quadratic statistic $\vT(\vtheta)^\top = (\vtheta,\vtheta \vtheta^\top)$ yields the bounded-support and heavy-tailed $q$-Gaussian families used by the $q$-VON \cref{alg:qvon,alg:tvon}.

\subsection{Generalizing the Variational Objective}
Using FYDs and the generalized maxent families, the generalized  BLR objective  
\begin{align}\label{eq:fy_obj}
    \min_{\text{\vlambda} \in \Lambda } ~ \myexpect_{\text{\vtheta} \sim \varpost} \left[ \ell(\vparam) \right] + \mathbb{D}_\Omega(\eta_{\text{\vlambda}_0}, \varpost), \ \
    \Lambda = \left\{ \vlambda \in \real^P: \argmax_{\pi \in\mathcal P} \left( \myexpect_{\pi}[\eta_{\text{\vlambda}}(\vtheta)]-\Omega(\pi) \right) \neq \varnothing \right\},
\end{align}
recovers the BLR objective when $\Omega = -\mathcal{H}$ and $ \eta_{\text{\vlambda}_0} = \log \prior$.
Other choices of $\Omega$ amount to new objectives. 
From the definition of FYDs we can write \cref{eq:fy_obj} as
\begin{align}\label{eq:fy_obj2}
 \min_{\text{\vlambda} \in \Lambda} ~ \myexpect_{\text{\vtheta} \sim \varpost} \left[   \ell(\vparam) -\eta_{\text{\vlambda}_0}(\vtheta) \right] + \Omega(\varpost). 
\end{align}
This separation in terms of a linear expectation of $\eta_{\text{\vlambda}_0}$ and a generalized negentropy of $\varpost$ is not possible with $f$-divergences (except the KLD). 
Up to a $\varpost$-independent scalar \cref{eq:fy_obj} is the FYD \smash{\(\mathbb D_\Omega(\eta_{\text{\vlambda}_0}-\ell,\varpost)\)}.
Thus, observing a loss $\ell$ tilts the reference score from \(\eta_{\text{\vlambda}_0}\) to \(\eta_{\text{\vlambda}_0}-\ell\). Since losses compose associatively in dual space, processing losses sequentially or jointly results in the same score, preserving sequential coherence in score-space. 

\subsection{Generalizing the BLR}

\begin{thm}[Fenchel-Young BLR]\label{thm:fy_blr}
Let \(\mathcal Q_{\Omega, \text{\vT}}\), \(\Psi\), \(\Psi^*\), and the moment map \( \vmu(\vlambda)=\myexpect_{\pi_{\text{\vlambda}}}[\vT(\vtheta)] \) be as in \Cref{ssec:maxent,prop:psi_star,prop:score_optimality}. Let \(\Lambda_0\) be an open coordinate chart of the induced family \(\mathcal Q_{\Omega,\text{\vT}}\). Assume that \( \Psi^*\in C^2(\Lambda_0)\), and define a generalized Fisher\footnote{Note that this is different from the ``$q$-Fisher'' discussed by \citet{amari2012geometry} who consider instead the Hessian of the cumulant of the solution to an escort-expectation constrained variant of \cref{eq:maxent_psi}, which leads to a different geometry than ours when $\postdeformidx \ne 1$.}   \(\vF(\vlambda):=\nabla^2_{\text{\vlambda}}\Psi^*(\vlambda) \) invertible on \(\Lambda_0\). Assume also that the dual potential \(\Psi\) is differentiable on \(\vmu(\Lambda_0)\), with \( \nabla_{\text{\vmu}}\Psi(\vmu(\vlambda))=\vlambda . \) Fix \(\vlambda_0\in\Lambda_0\), and define \( \bar\ell(\vtheta)=\ell(\vtheta)-\eta_{\text{\vlambda}_0}(\vtheta) \). Assume that the map \(\vmu\mapsto \myexpect_{\pi_{\text{\vlambda}(\text{\vmu})}}[\bar\ell(\vtheta)] \) is differentiable on \(\vmu(\Lambda_0)\). Then the generalized canonical gradient satisfies
\begin{align}
\vF(\vlambda)^{-1}
\nabla_{\text{\vlambda}}
\left[ \myexpect_{\pi_{\text{\vlambda}}}[\ell(\vtheta)] + \mathbb D_\Omega(\eta_{\text{\vlambda}_0},\pi_{\text{\vlambda}})
\right] 
= \nabla_{\text{\vmu}} \myexpect_{\pi_{\text{\vlambda}(\text{\vmu})}}[\bar\ell(\vtheta)] \Big|_{\text{\vmu}=\text{\vmu}(\text{\vlambda})} + \vlambda .
\label{eq:fy_blr_gradient_identity}
\end{align}
Consequently, the canonical-gradient step is
\begin{align}
\vlambda_{t+1} \gets
\vlambda_t - \rho_t \vF(\vlambda_t)^{-1} \nabla_{\text{\vlambda}} \left( \underset{{\pi_{\text{\vlambda}}}}{\myexpect}  [\ell(\vtheta)] + \mathbb D_\Omega(\eta_{\text{\vlambda}_0},\pi_{\text{\vlambda}}) \right) \Big|_{\text{\vlambda}_t}
= (1-\rho_t)\vlambda_t - \rho_t \nabla_{\text{\vmu}} \underset{{\pi_{\text{\vlambda}(\text{\vmu})}}}{\myexpect} [\bar\ell(\vtheta)] \Big|_{\text{\vmu}_t} .
 \label{eq:non_blr}
\end{align}
In particular, any fixed point with \(\rho_t>0\) satisfies \( \vlambda^* = - \nabla_{\text{\vmu}} \myexpect_{\pi_{\text{\vlambda}(\text{\vmu})}}[\bar\ell(\vtheta)] \Big|_{\text{\vmu}=\text{\vmu}(\text{\vlambda}^*)} \).
\end{thm}
The proof is in App.~\ref{app:thm1_proof}.
The canonical parameter at convergence is matched with the gradient of the loss. Analogous to the fixed-point interpretation of the BLR, this shows that generalized canonical gradients are intrinsic to the variational problem, rather than an optimization choice. The update in \Cref{eq:non_blr} is one of our main contributions. We call it the \textbf{generalized BLR}, as it applies to families of distributions larger than the exponential family (see~\Cref{fig:fig}). When $\Omega = - \mathcal{H}$ we recover the BLR. 
If \(\eta_{\text{\vlambda}_0}(\vtheta)-\ell(\vtheta) =\eta_{\text{\vlambda}'}(\vtheta)+\mathrm{const.}\),
then
\( 
\vlambda_{t+1}=(1-\rho_t)\vlambda_t+\rho_t\vlambda',
\)
so a unit step recovers the generalized conjugate posterior induced by the tilted score.


\section{Example: Tsallis Negentropy and $q$-VON}\label{sec:qvon_summary}

\begin{figure*}[t]
\noindent
\begin{minipage}[t]{\dimexpr0.5\textwidth-0.5\columnsep\relax}
\begin{algorithm}[H]
\caption{$\postdeformidx$-VON (bounded-support $\postdeformidx<1$)}
\label{alg:qvon}
\begin{algorithmic}[1]
\REQUIRE Learning rates $\rho_t$, weight-decay $\delta>0$
\REQUIRE Momentum $\beta_1,\beta_2\in[0,1)$
\REQUIRE Hessian init $h_0>0$, \\ \qquad \ \ \blue{Deformation index $\postdeformidx<1$}
\COMPUTE \blue{ $\kappa(\postdeformidx,D) \gets \frac{R^2}{D + 2(2-\postdeformidx)/(1-\postdeformidx)}$, \\
                \qquad \ \ \ \  {$\gamma \gets D + 2/(1-\postdeformidx) $}} 
\vspace{0.13cm}
\INIT $\vm \gets$ (NN-weights),   $\vh \gets h_0$, $\vg \gets 0$, $\zeta \gets N$
\vspace{0.13cm}
\INIT $\vs \gets  \zeta(\vh + \delta)$ \hfill (diagonal of $\vS$)
\INIT \smash{\blue{$\vsigma^2 \gets \left(  s_1 \cdot \hdots \cdot s_D \right)^{1/\gamma} / \vs$}} \hfill (scale)
\FOR{$t=1,2,\ldots$}
  \STATE $\widehat \vg \gets \nabla_{\text{\vtheta}} \ell(\vtheta)$, \hfill $\vtheta \sim \blue{\Qgauss(\vm,\mathrm{diag}(\vsigma^2))}$
  \STATE $\widehat \vh \gets \blue{ \kappa^{-1}} \widehat \vg \cdot (\vtheta-\vm) / \vsigma^2$ \\
     (estimator for $\myexpect_{\blue{\tilde \pi}}[\text{diag}(\nabla^2 \ell)]$ )
  \STATE $\vg \gets \beta_1 \vg + (1-\beta_1)\widehat \vg$
  \STATE $\vh \gets \beta_2 \vh + (1-\beta_2)\widehat \vh$
  \STATE $\vm \gets \vm - \rho_t (\vg+\delta\vm) / (\vh+\delta)$
  \STATE $\vs \gets \zeta(\vh + \delta)$
  \STATE \smash{\blue{$\vsigma^2 \gets \left( s_1 \cdot \hdots \cdot s_D \right)^{1/\gamma} / \vs$}}
\ENDFOR
\STATE \textbf{return} $\vm,\vsigma^2$
\end{algorithmic}
\end{algorithm}
\end{minipage}\hfill
\begin{minipage}[t]{\dimexpr0.5\textwidth-0.5\columnsep\relax}
\begin{algorithm}[H]
\caption{Student-$t$-VON (heavytailed $\postdeformidx>1$)}
\label{alg:tvon}
\begin{algorithmic}[1]
\REQUIRE Learning rates $\rho_t$, weight-decay $\delta>0$
\REQUIRE Momentum $\beta_1,\beta_2\in[0,1)$
\REQUIRE Hessian init $h_0>0$, \\ \qquad \ \  \blue{Degrees of freedom $\nu>2$}
\COMPUTE  \blue{\\ \qquad $\kappa_t \gets \frac{\nu+D}{\nu}\left[(\nu+D)^{\frac{D}{2}}\pi^{\frac{D}{2}}\frac{\Gamma(\frac{\nu}{2})}{\Gamma(\frac{\nu+D}{2})}\right]^{\frac{2}{\nu}}$}
\INIT $\vm \gets$ (NN-weights),  $\vh \gets h_0$,  $\vg \gets 0$, $\zeta \gets N$
\INIT $\vs \gets  \zeta(\vh + \delta)$ \hfill (diagonal of $\vS$)
\INIT \smash{\blue{$\vsigma^2 \gets \kappa_t \left( s_1 \cdot \hdots \cdot s_D \right) ^{-1/\nu} / \vs$}} \hfill (scale)
\FOR{$t=1,2,\ldots$}
  \STATE $\widehat \vg \gets \nabla_{\text{\vparam}} \ell(\vparam)$, \hfill $\vparam \sim \blue{\mathcal{T}_{\nu}(\vm,\mathrm{diag}(\vsigma^2))}$
  \STATE $\widehat \vh \gets \blue{\frac{\nu-2}{\nu}} \widehat \vg \cdot (\vtheta-\vm) / \vsigma^2$  \\ (estimator for $\myexpect_{\blue{\tilde \pi}}[\text{diag}(\nabla^2 \ell)]$)
  \STATE $\vg \gets \beta_1 \vg + (1-\beta_1)\widehat \vg$
  \STATE $\vh \gets \beta_2 \vh + (1-\beta_2)\widehat \vh$
  \STATE $\vm \gets \vm - \rho_t\,(\vg+\delta\vm) / (\vh+\delta)$
  \STATE $\vs \gets \zeta(\vh+\delta)$
  \STATE \smash{\blue{$\vsigma^2 \gets \kappa_t \left( s_1 \cdot \hdots \cdot s_D \right) ^{-1/\nu} / \vs$}}
\ENDFOR
\STATE \textbf{return} $\vm,\vsigma^2$
\end{algorithmic}
\end{algorithm}
\end{minipage}
\end{figure*}

As our main example we replace the Shannon negentropy by the Tsallis negentropy $\Omega_{2-\postdeformidx}$ and derive \Cref{alg:qvon,alg:tvon} (see App.~\ref{sec:qvon} for details), which we apply to CIFAR-100, CIFAR-10, MNIST (see App.~\ref{app:experiments}), and to Bayesian logistic regression on a UCI dataset (\cref{fig:fig3}). The induced maxent family is the $\postdeformidx$-exponential family (see App.~\ref{ssec:tsallis_maxent} for a proof),
\begin{align}
    \varpost(\vtheta) =
    \exp_{\postdeformidx} \left(\vlambda^\top \vT(\vtheta) - A_{2-\postdeformidx}(\vlambda)\right),
    \label{eq:deformed_exp_families}
\end{align}
where $\exp_{\postdeformidx}(u) =  [1 + (1-\postdeformidx) u]_+^{1/(1-\postdeformidx)} $ when $q\neq1$. We take \(\Theta=\mathbb R^D\), with Lebesgue reference
measure, so \(\vtheta\in\mathbb R^D\). For the quadratic statistic $\vT(\vtheta)^\top=(\vtheta,\vtheta\vtheta^\top)$ with canonical parameter $\vlambda^\top=(\vS\vm,-\tfrac12\vS)$, this yields a location-scale $(\vm,\vSigma)$ family of $\postdeformidx$-Gaussians. The quadratic geometry is unchanged, but $\postdeformidx$ controls support and tail shapes: $\postdeformidx<1$ gives parameter-dependent bounded-support \cref{alg:qvon}, and $\postdeformidx>1$ gives heavy-tailed \cref{alg:tvon} with $\nu=2/(\postdeformidx-1)-D$. 
In App.~\ref{app:q_gauss} we review the canonical, location-scale, and expectation parametrizations of bounded-support and heavy-tailed $q$-Gaussians, as well as a closed-form FYD for the heavy-tailed case, which we use to evaluate the objective optimized by \cref{alg:tvon}. In App.~\ref{app:bonnet_price_stein} we derive the reparameterized gradient estimators underlying lines 2-3 of \Cref{alg:qvon,alg:tvon}.

Specializing the generalized BLR update~\eqref{eq:non_blr} to quadratic families with expectation parameters $(\vmu,\vM)=(\myexpect_{\varpost}[\vtheta],\myexpect_{\varpost}[\vtheta\vtheta^\top])$, we obtain Newton-like canonical updates. After converting gradients with respect to $(\vmu,\vM)$ to gradients with respect to $(\vm,\vSigma)$ and applying Bonnet- and Price-type identities (App.~\ref{app:bonnet_price_stein}), both the bounded-support and heavy-tailed cases take the common form (App.~\ref{sec:qvon})
\begin{align}
    \vS_{t+1} = (1-\rho_t)\vS_t + \rho_t \myexpect_{\tilde{\pi}_t^{2-\postdeformidx}} \left[\nabla_{\text{\vtheta}}^2 \bar\ell(\vtheta)\right], \ \ \
    \vm_{t+1} = \vm_t - \rho_t \vS_{t+1}^{-1} \myexpect_{\pi_t}\left[\nabla_{\text{\vtheta}}\bar\ell(\vtheta)\right],
    \label{eq:qvon}
\end{align}
where $\tilde{\pi}^{2-\postdeformidx}$ denotes the $(2-\postdeformidx)$-escort of $\pi$. The updates in \cref{eq:qvon} are strikingly similar in form to those of VON and IVON \citep{khan2018fast,osawa2019practical,shen2402variational}. The difference between $\postdeformidx<1$ and $\postdeformidx>1$ lies in the relation between $(\vS,\vm)$ and the location-scale parametrizations, and in the corresponding escort reweighting appearing in the Price identity. The complete derivations, including the bounded-support and Student-$t$ closed forms, and the escort-based Hessian identities are given in Apps.~\ref{sec:qvon}, \ref{app:q_gauss} and \ref{app:bonnet_price_stein}. \Cref{alg:qvon,alg:tvon} give practical diagonal implementations of these updates. Relative to IVON, only the sampling law, the constants relating $\vS$ to the scale, and the corresponding diagonal Stein estimator change. As in IVON, we use minibatch stochastic gradients, a diagonal precision $\vS=\mathrm{diag}(\vs)$, the reparametrization $\vs=\zeta(\vh+\delta)$, and momentum terms for stable deep-learning optimization. While the CIFAR-100 experiment we show in App.~\ref{app:experiments} is not a large-scale benchmark, there is no structural obstacle to scaling further, and applying $q$-VON to transformer- or LLM-scale training is an empirical engineering question, which we leave to future work.


\paragraph{Mean-Field.} In high-dimensional models we use a factorized mean-field $q$-VON-MF, where each parameter follows a 1D $q$-Gaussian. We replace $\Omega$ in~\cref{eq:fy_obj} with a sum of 1D terms. The only change to \Cref{alg:qvon,alg:tvon} is to use 1D constants (set $D=1$), and update each variance parameter separately rather than through the product \smash{$s_1\cdots s_D$} (App.~\ref{ssec:meanfield}). This addresses the problem that finite covariance in high-dimensional settings forces heavy-tailed $q$-values to be close to 1, and that $D$ can dominate the max support radius in bounded-support settings. An exciting future direction that our work opens is training deep networks whose parameters have variable levels of stochasticity; for example, each layer may be associated with distinct, learnable or tunable support- and tail-behaviors.

\section{Experiments}\label{sec:experiments_summary}

\begin{figure*}[ht] 
    \begin{center}
    \includegraphics[width=\textwidth]{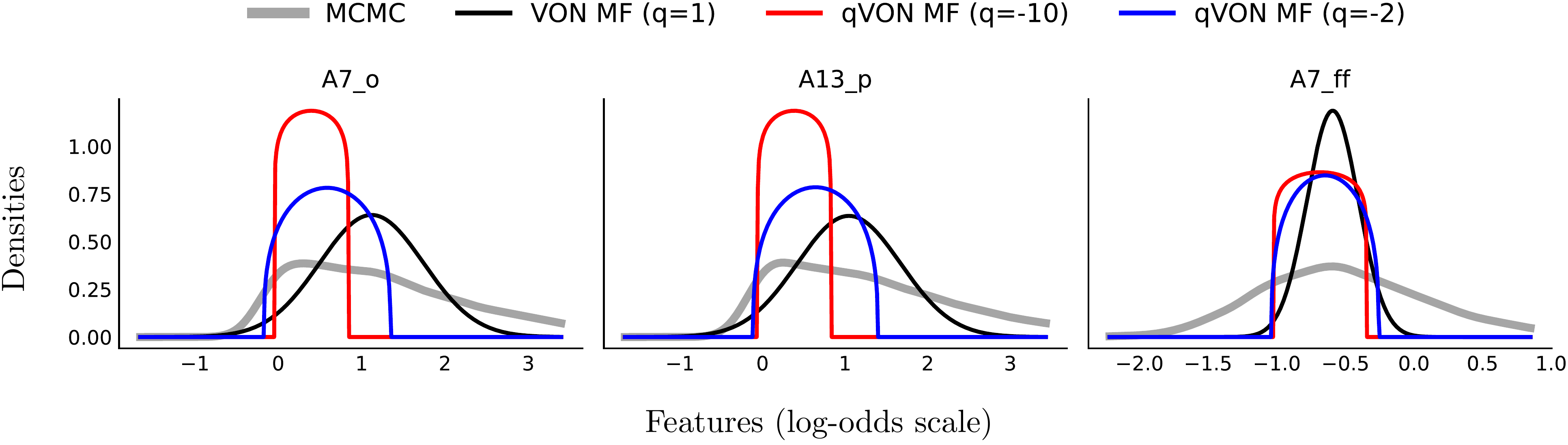}
    \caption{Bounded-support, mean-field variational posterior (red, blue) means can approach the posterior mode of features A7\_0, A13\_p and A7\_ff, by changing the value of $q$ in Bayesian logistic regression on the UCI Credit Approval dataset. This is useful when posterior marginals (grey) are skewed, and thus the empirical mode is badly approximated by a Gaussian posterior (black). }
    \label{fig:fig3}
\end{center}
\end{figure*}

We evaluate $q$-VON on Bayesian logistic regression and image classification. Detailed experiment descriptions are in App.~\ref{app:experiments}, and results on varying $q$ are in App.~\ref{app:varying_q}. 
On Bayesian logistic regression, shrinking the max support radius by decreasing $\postdeformidx<1$ lets the $q$-VON approximation closely track the MCMC posterior mode, in contrast to the Gaussian approximation (\Cref{fig:fig3}). 
On MNIST, CIFAR-10, and CIFAR-100 we compare $q$-VON and $q$-VON-MF against SGD, VSGD, IVON, and a mean-only baseline $q$-VSGD, using test accuracy, NLL, ECE, Brier score, and AUROC. Generally, $q$-VON improves over SGD and $q$-VSGD, showing that learning the posterior shape outperforms noise injection. The choice of tail behavior can trade off generalization and calibration: on MNIST, heavy-tailed $q$-VON has the highest accuracy while bounded-support $q$-VON has the best calibration. On CIFAR, $q$-VON-MF is competitive with (highly tuned) IVON without additional tuning. 
Our experiments demonstrate that generalized BLR updates can be implemented at the scale of modern neural-network optimizers while extending IVON to posterior families unavailable to the BLR. 

\bibliography{refs}






\newpage
\appendix



\startcontents[appendices]
\printcontents[appendices]{}{1}[2]{\section*{Appendix Contents}}

\newpage

\section{The sharp Gibbs bound forces the KLD}\label{app:kld}

\begin{prop}
Fix a prior density $\prior$ with respect to $h$. Suppose $\mathbb D(\varcand,\prior)$ is a loss-independent functional such that for every measurable loss $\loss:\Theta\to\real\cup\{+\infty\}$ with
\[
0 < \int e^{-\loss(\text{\vtheta})}\prior(\vtheta) dh(\vtheta) < \infty
\]
and every $\varcand \in \mathcal P$,
\[
-\log \int e^{-\loss(\text{\vtheta})}\prior(\vtheta) dh(\vtheta)
\le
\myexpect_{\varcand}[\loss(\vtheta)] + \mathbb D(\varcand,\prior),
\]
with equality at the Gibbs density
\[
\varcand_{\loss}(\vtheta) = \frac{e^{-\loss(\text{\vtheta})} \prior(\vtheta)}
{\int e^{-\loss(\text{\vtheta}')} \prior(\vtheta') dh(\vtheta')}.
\]
Then for every $\varcand \ll \prior$ such that $\dkl{\varcand}{\prior} < \infty$, defining
\[
\loss_{\varcand}(\vtheta) = -\log\frac{\varcand(\vtheta)}{\prior(\vtheta)},
\]
we have
\[
\mathbb D(\varcand,\prior) = \dkl{\varcand}{\prior}.
\]
\end{prop}

\begin{proof}
We choose a loss whose Gibbs posterior is the candidate posterior we want, and then evaluate the variational identity there. Let $\loss = \loss_{\varcand}$. Then
\[
\int e^{-\loss_{\varcand}(\text{\vtheta})} \prior(\vtheta) dh(\vtheta)
= \int \varcand(\vtheta) dh(\vtheta) = 1,
\]
so the left-hand side of the inequality is $0$. Hence
\[
0 = \myexpect_{\varcand} \left[-\log\frac{\varcand(\vtheta)}{\prior(\vtheta)}\right] + \mathbb D(\varcand,\prior)
= -\dkl{\varcand}{\prior} + \mathbb D(\varcand,\prior).
\]
\end{proof}

\section{Maxent}\label{app:maxent}

\subsection{Expfam as max Shannon entropy}\label{ssec:shannon_maxent}
The following proof is based on \citet[Prop. 9.1.2]{lasota2013chaos}. 
\begin{prop}
Let $(\Theta, \mathcal{A}, h)$ be a measure space, and let
\[
\vT(\vtheta) = \left(T_1(\vtheta), ..., T_P(\vtheta)\right)
\]
be a measurable vector-valued function. Assume there exist
\[
\vmu = (\mu_1, ..., \mu_P) \qquad \text{and} \qquad \vlambda = (\lambda_1, ..., \lambda_P)
\]
such that
\[
\mu_i = \frac{\int_{\Theta} T_i(\vtheta)\exp \left(\langle \vlambda, \vT(\vtheta)\rangle \right) dh(\vtheta)}
{\int_{\Theta} \exp \left(\langle \vlambda, \vT(\vtheta) \rangle \right) dh(\vtheta)},
\qquad i=1, ..., P,
\]
and all of the integrals are finite. Then the maximum of the Shannon entropy
\[
\mathcal H(\pi) = -\int_{\Theta} \pi(\vtheta) \log \pi(\vtheta) dh(\vtheta)
\]
over all densities $\pi \in \mathcal{P}$ satisfying
\[
\mu_i = \int_{\Theta} T_i(\vtheta) \pi(\vtheta) dh(\vtheta),
\qquad i=1, ..., P,
\]
is attained at
\[
\pi_{\text{\vlambda}}(\vtheta) = \frac{\exp \left(\langle \vlambda, \vT(\vtheta) \rangle \right)}
{\int_{\Theta} \exp \left(\langle \vlambda, \vT(\vtheta) \rangle \right) dh(\vtheta)}.
\]
\end{prop}

\begin{proof}
Set
\[
Z = \int_{\Theta} \exp \left(\langle \vlambda, \vT(\vtheta)\rangle \right) dh(\vtheta),
\qquad 
\pi_{\text{\vlambda}}(\vtheta) = Z^{-1}\exp \left(\langle \vlambda, \vT(\vtheta) \rangle \right).
\]
Let $\pi \in \mathcal{P}$ satisfy
\[
\mu_i = \int_{\Theta} T_i(\vtheta)\pi(\vtheta) dh(\vtheta),
\qquad i=1, ..., P.
\]
Then
\begin{align*}
\mathcal H(\pi) &= -\int_{\Theta} \pi(\vtheta) \log \pi(\vtheta) dh(\vtheta) \\
&= -\int_{\Theta} \pi(\vtheta) \log \pi_{\text{\vlambda}}(\vtheta) dh(\vtheta)
-\int_{\Theta} \pi(\vtheta) \log \frac{\pi(\vtheta)}{\pi_{\text{\vlambda}}(\vtheta)} dh(\vtheta) \\
&\le -\int_{\Theta} \pi(\vtheta) \log \pi_{\text{\vlambda}}(\vtheta) dh(\vtheta),
\end{align*}
since
\[
\int_{\Theta}\pi(\vtheta) \log \frac{\pi(\vtheta)}{\pi_{\text{\vlambda}}(\vtheta)} dh(\vtheta) \ge 0.
\]
Now
\begin{align*}
-\int_{\Theta}\pi(\vtheta)\log \pi_{\text{\vlambda}}(\vtheta)\,dh(\vtheta)
&= -\int_{\Theta} \pi(\vtheta)
\left(\langle \vlambda, \vT(\vtheta) \rangle - \log Z \right) dh(\vtheta) \\
&= \log Z - \sum_{i=1}^P \lambda_i \int_{\Theta} T_i(\vtheta) \pi(\vtheta) dh(\vtheta) \\
&= \log Z - \sum_{i=1}^P \lambda_i \mu_i.
\end{align*}
Therefore,
\[
\mathcal H(\pi) \le \log Z - \sum_{i=1}^P \lambda_i\mu_i.
\]

It remains to show that $\pi_{\text{\vlambda}}$ attains this value. We compute
\begin{align*}
\mathcal H(\pi_{\text{\vlambda}})
&= -\int_{\Theta} \pi_{\text{\vlambda}}(\vtheta) \log \pi_{\text{\vlambda}}(\vtheta) dh(\vtheta) \\
&= -\int_{\Theta}\pi_{\text{\vlambda}}(\vtheta)
\left(\langle \vlambda, \vT(\vtheta)\rangle - \log Z \right) dh(\vtheta) \\
&= \log Z-\sum_{i=1}^P \lambda_i \int_{\Theta} T_i(\vtheta) \pi_{\text{\vlambda}}(\vtheta) dh(\vtheta).
\end{align*}
By the definition of $\pi_{\text{\vlambda}}$,
\[
\int_{\Theta} T_i(\vtheta)\pi_{\text{\vlambda}}(\vtheta) dh(\vtheta)
= \frac{\int_{\Theta} T_i(\vtheta) \exp \left(\langle \vlambda, \vT(\vtheta) \rangle \right) dh(\vtheta)}
{\int_{\Theta} \exp \left(\langle \vlambda, \vT(\vtheta) \rangle \right) dh(\vtheta)}
= \mu_i.
\]
Hence
\[
\mathcal H(\pi_{\text{\vlambda}}) = \log Z - \sum_{i=1}^P \lambda_i \mu_i.
\]
So for every feasible $\pi$,
\[
\mathcal H(\pi)\le \mathcal H(\pi_{\text{\vlambda}}),
\]
which proves that $\pi_{\text{\vlambda}}$ maximizes the entropy.
\end{proof}

\subsection{$q$-expfam as Max Tsallis $(2-q)$-entropy}\label{ssec:tsallis_maxent}

\begin{prop}
Let $(\Theta, \mathcal{A}, h)$ be a measure space, and let
\[
\vT(\vtheta) = \left(T_1(\vtheta), ..., T_P(\vtheta) \right)
\]
be measurable. Fix $\postdeformidx < 2$, $\postdeformidx \neq 1$, and define
\[
\beta_{\postdeformidx}(t)
= \begin{cases}
\displaystyle
\frac{t^{2-\postdeformidx}-t}{(2-\postdeformidx)(1-\postdeformidx)},
& t \ge 0, \\
+ \infty, & t < 0
\end{cases}
\]
and
\[
\Omega_{2-\postdeformidx}(\pi) = \int_\Theta \beta_{\postdeformidx} (\pi(\vtheta)) dh(\vtheta).
\]
On the simplex of densities,
\[
\Omega_{2-\postdeformidx}(\pi) = -\frac{1}{2-\postdeformidx}\mathcal S_{2-\postdeformidx}(\pi),
\qquad
\mathcal{S}_{2-\postdeformidx}(\pi) = \frac{1 - \int_\Theta \pi(\vtheta)^{2-\postdeformidx} d h(\vtheta)}{1-\postdeformidx}.
\]

Assume there exist $\vmu = (\mu_1, ...,\mu_P)$ and $\vlambda = (\lambda_1, ..., \lambda_P)$ such that
\[
\pi_{\text{\vlambda}} (\vtheta) =
\exp_{\postdeformidx} \left(\langle \vlambda,\vT(\vtheta) \rangle - A_{2-\postdeformidx}(\vlambda) \right)
\]
is a density, $\Omega_{2-\postdeformidx}(\pi_{\vlambda}) < \infty$, all the integrals below are finite, and
\[
\mu_i = \int_\Theta T_i(\vtheta) \pi_{\text{\vlambda}}(\vtheta) dh(\vtheta),
\qquad i = 1, ..., P.
\]
Then $\pi_{\text{\vlambda}}$ minimizes $\Omega_{2-\postdeformidx}$, equivalently maximizes $\mathcal{S}_{2-\postdeformidx}$, over all densities $\pi \in \mathcal{P}$ satisfying
\[
\mu_i = \int_\Theta T_i(\vtheta) \pi(\vtheta) dh(\vtheta),
\qquad i = 1, ..., P.
\]
\end{prop}
\begin{proof}
Set
\[
\eta_{\text{\vlambda}}(\vtheta) = \langle \vlambda, \vT(\vtheta) \rangle - A_{2-\postdeformidx}(\vlambda).
\]
Let $\pi \in \mathcal{P}$ satisfy the same moment constraints as $\pi_{\text{\vlambda}}$.

We claim that for $h$-a.e. $\vtheta$,
\[
\eta_{\text{\vlambda}}(\vtheta) + \frac{1}{2-\postdeformidx} \in \partial \beta_{\postdeformidx} \left(\pi_{\text{\vlambda}} (\vtheta) \right).
\]

If $\pi_{\text{\vlambda}}(\vtheta) > 0$, then
\[
\log_{\postdeformidx}\pi_{\text{\vlambda}} (\vtheta) = \eta_{\text{\vlambda}} (\vtheta),
\]
and since
\[
\beta_{\postdeformidx}'(t) = \log_{\postdeformidx}(t) + \frac{1}{2-\postdeformidx}, \qquad t > 0,
\]
it follows that
\[
\eta_{\text{\vlambda}}(\vtheta) + \frac{1}{2-\postdeformidx} = \beta'_{\postdeformidx} \left(\pi_{\text{\vlambda}}(\vtheta) \right) \in \partial \beta_{\postdeformidx} \left(\pi_{\text{\vlambda}} (\vtheta) \right).
\]

If $\postdeformidx < 1$ and $\pi_{\text{\vlambda}} (\vtheta) = 0$, then by the definition of $\exp_{\postdeformidx}$ we have
\[
\eta_{\text{\vlambda}}(\vtheta) \le -\frac{1}{1-\postdeformidx}.
\]
Also, the right derivative at 0 is
\[
\beta'_{\postdeformidx}(0) = -\frac{1}{(1-\postdeformidx) (2-\postdeformidx)},
\qquad 
\partial \beta_{\postdeformidx}(0) = (-\infty, \beta'_{\postdeformidx}(0)].
\]
Therefore
\[
\eta_{\text{\vlambda}}(\vtheta)+\frac{1}{2-\postdeformidx} \le - \frac{1}{1-\postdeformidx} + \frac{1}{2-\postdeformidx}
= -\frac{1}{(1-\postdeformidx)(2-\postdeformidx)}
= \beta'_{\postdeformidx}(0),
\]
and again
\[
\eta_{\text{\vlambda}}(\vtheta) + \frac{1}{2-\postdeformidx} \in \partial \beta_{\postdeformidx} \left(\pi_{\text{\vlambda}}(\vtheta)\right).
\]

Thus, by the subgradient inequality,
\[
\beta_{\postdeformidx} (\pi(\vtheta)) \ge 
\beta_{\postdeformidx}(\pi_{\text{\vlambda}} (\vtheta)) + \left(\eta_{\text{\vlambda}}(\vtheta) + \frac{1}{2-\postdeformidx} \right) \left(\pi(\vtheta) - \pi_{\text{\vlambda}} (\vtheta) \right)
\]
for $h$-a.e. $\vtheta$. Integrating we get
\begin{align*}
\Omega_{2-\postdeformidx}(\pi) - \Omega_{2-\postdeformidx}(\pi_{\text{\vlambda}})
&\ge \int_\Theta \left(\eta_{\text{\vlambda}}(\vtheta)+\frac{1}{2-\postdeformidx} \right)
\left(\pi(\vtheta)-\pi_{\text{\vlambda}}(\vtheta) \right) dh(\vtheta) \\
&= \int_\Theta \eta_{\text{\vlambda}}(\vtheta) \left(\pi(\vtheta) - \pi_{\text{\vlambda}} (\vtheta) \right) d h(\vtheta),
\end{align*}
because both $\pi$ and $\pi_{\text{\vlambda}}$ integrate to $1$.

Expand $\eta_{\text{\vlambda}}$:
\begin{align*}
&\int_\Theta \eta_{\text{\vlambda}}(\vtheta)
\left(\pi(\vtheta) - \pi_{\text{\vlambda}} (\vtheta) \right) dh(\vtheta) \\
&= \sum_{i=1}^P \lambda_i \left( \int_\Theta T_i(\vtheta)\pi(\vtheta) dh(\vtheta) - \int_\Theta T_i(\vtheta) \pi_{\text{\vlambda}}(\vtheta) dh(\vtheta) \right) 
- A_{2-\postdeformidx} (\vlambda) \left( \int_\Theta \pi(\vtheta) dh(\vtheta) - \int_\Theta \pi_{\text{\vlambda}}(\vtheta) dh(\vtheta) \right) 
= 0,
\end{align*}
since $\pi$ and $\pi_{\text{\vlambda}}$ have the same prescribed moments and both are densities. Hence
\[
\Omega_{2-\postdeformidx} (\pi) \ge \Omega_{2-\postdeformidx} (\pi_{\text{\vlambda}}).
\]

When $\postdeformidx < 2$,
\[
\Omega_{2-\postdeformidx}(\pi) = -\frac{1}{2 - \postdeformidx} \mathcal{S}_{2-\postdeformidx} (\pi),
\]
and minimizing $\Omega_{2-\postdeformidx}$ is equivalent to maximizing $\mathcal S_{2-\postdeformidx}$.
\end{proof}

This certifies optimality of a feasible $\postdeformidx$-exponential candidate. What we have not proven is the existence of $\vlambda$ or of a feasible $\pi_{\text{\vlambda}}$.

\textit{Remark. } In general we assume
\[
\postdeformidx < \frac{D+4}{D+2}, \qquad \postdeformidx \neq 1.
\]
Then $\postdeformidx < 1$ yields bounded-support $q$-Gaussians, while
\[
1 < \postdeformidx < \frac{D+4}{D+2}
\]
yields Student-$t$ $q$-Gaussians with
\[
\nu = \frac{2}{\postdeformidx - 1} - D > 2.
\]
Under the parameter range imposed by finite covariance, $\postdeformidx = 2$ is excluded.

\subsection{Existence}\label{ssec:maxent_existence}

Our abstract construction in Sec.~\ref{ssec:maxent} is stated on the canonical domain
\[
\Lambda_\Omega = \{\vlambda\in\real^P: \Omega^*(\eta_{\text{\vlambda}}) \text{ is attained by a unique } \pi_{\text{\vlambda}}\in\mathcal P\}.
\]
This conditional formulation is deliberate. A general existence theorem for arbitrary \(\Omega\) would require additional infinite-dimensional assumptions on the support, the moment functions, and the regularizer. For the Shannon case, sufficient conditions of this kind are available in \citet{ishwar2005existence}. For the Tsallis families used in Sec.~\ref{sec:qvon}, we do not need such a general result because the optimizer is explicit: the bounded-support case is compactly supported, while the heavy-tailed case reduces to a Student-\(t\) family whose normalization and moment conditions can be checked directly from the closed form in App.~\ref{app:q_gauss}.

\section{Proof of \cref{thm:fy_blr}}\label{app:thm1_proof}

To generalize \citet{khan2023bayesian} we derive a canonical gradient descent for optimizing \cref{eq:fy_obj2}. We use \smash{$\barloss(\text{\vtheta}) =  \ell(\vparam) -\eta_{\text{\vlambda}_0}(\vtheta) $}. 
By \Cref{prop:psi_star}, the finite-dimensional potential $\Psi^*$ characterizes the geometry of the canonical parameter space. Thus, we define the generalized Fisher matrix, 
\begin{align}
\vF(\vlambda) = \nabla^2_{\text{\vlambda}} \Psi^*(\vlambda)
= \nabla^2_{\text{\vlambda}} \Omega^*(\eta_{\text{\vlambda}})
= \frac{\partial \text{\vmu}(\text{\vlambda})}{\partial \text{\vlambda}}.
\end{align}
and use the chain rule $\vF(\vlambda)^{-1} \nabla_{\text{\vlambda}}   = \vF(\vlambda)^{-1} \vF(\vlambda) \nabla_{\text{\vmu}}  = \nabla_{\text{\vmu}} $, to obtain a generalized canonical gradient update of the same form as that of \citet{khan2023bayesian}. Whenever $\vF(\vlambda)$ exists and is invertible, the Fenchel-Young BLR is  
\begin{align}
    \vlambda_{t+1} &\gets \vlambda_t - \rho_t \vF(\vlambda_t)^{-1} \nabla_{\text{\vlambda}} \left( \underset{\varpost}{\myexpect} [\bar{\ell}(\vtheta)] + \Omega(\varpost)\right) \Big|_{{\text{\vlambda}} = {\text{\vlambda}}_t} 
    = (1 - \rho_t) \vlambda_t - \rho_t \nabla_{\text{\vmu}} \myexpect_{\varpost}[\bar{\ell}(\vtheta)].
\end{align}
The equality follows from the fact that \smash{$\nabla_{\text{$\vmu$}} \Omega(\pi_{\text{$\vlambda$}}) = \nabla_{\text{$\vmu$}} \Psi(\vmu) = \vlambda$} (\Cref{ssec:maxent}), which results in the $(1-\rho_t)$ factor. This characterizes the stationary condition of the generalized variational objective. At any fixed point $\vlambda^*$ of the canonical update \( \vlambda^* = - \nabla_{\text{\vmu}} \myexpect_{\pi_{\text{\vlambda}^*}}[\bar{\ell}(\vtheta)] \).

\begin{proof}[Proof of \cref{thm:fy_blr}]
By the definition of the Fenchel-Young discrepancy,
\begin{align}
\myexpect_{\pi_{\text{\vlambda}}}[\ell(\vtheta)] + \mathbb D_\Omega(\eta_{\text{\vlambda}_0},\pi_{\text{\vlambda}})
= \Omega^*(\eta_{\text{\vlambda}_0}) + \myexpect_{\pi_{\text{\vlambda}}}[\bar\ell(\vtheta)] + \Omega(\pi_{\text{\vlambda}}).
\label{eq:fy_blr_obj_expanded}
\end{align}
By \Cref{prop:psi_star}, \( \Omega^*(\eta_{\text{\vlambda}_0}) = \Psi^*(\vlambda_0)\), which is constant with respect to \(\vlambda\). By \Cref{prop:score_optimality}, each \(\pi_{\text{\vlambda}}\) solves the maxent problem at moment vector \(\vmu(\vlambda)\). Hence \( \Omega(\pi_{\text{\vlambda}}) = \Psi(\vmu(\vlambda)) \). Therefore, up to the additive constant \(\Psi^*(\vlambda_0)\),
\[
\myexpect_{\pi_{\text{\vlambda}}}[\ell(\vtheta)] + \mathbb D_\Omega(\eta_{\text{\vlambda}_0},\pi_{\text{\vlambda}})
= \myexpect_{\pi_{\text{\vlambda}}}[\bar\ell(\vtheta)] + \Psi(\vmu(\vlambda)).
\]
Taking the gradient w.r.t. \(\vlambda\) and applying the chain rule 
\begin{align}
\nabla_{\text{\vlambda}} \left[ \myexpect_{\pi_{\text{\vlambda}}}[\ell(\vtheta)] + \mathbb D_\Omega(\eta_{\text{\vlambda}_0},\pi_{\text{\vlambda}}) \right]
&= \frac{\partial \vmu(\vlambda)}{\partial \vlambda} \left[ \nabla_{\text{\vmu}} \myexpect_{\pi_{\text{\vlambda}(\text{\vmu})}}[\bar\ell(\vtheta)] + \nabla_{\text{\vmu}}\Psi(\vmu) \right]_{\text{\vmu}=\text{\vmu}(\text{\vlambda})}
\nonumber\\
&= \vF(\vlambda) \left[ \nabla_{\text{\vmu}} \myexpect_{\pi_{\text{\vlambda}(\text{\vmu})}}[\bar\ell(\vtheta)] + \nabla_{\text{\vmu}}\Psi(\vmu) \right]_{\text{\vmu}=\text{\vmu}(\text{\vlambda})} .
\end{align}
Using the dual-coordinate relation \( \nabla_{\text{\vmu}}\Psi(\vmu(\vlambda))=\vlambda \) and multiplying by \(\vF(\vlambda)^{-1}\) we get
\[
\vF(\vlambda)^{-1} \nabla_{\text{\vlambda}} \left[ \myexpect_{\pi_{\text{\vlambda}}}[\ell(\vtheta)] + \mathbb D_\Omega(\eta_{\text{\vlambda}_0},\pi_{\text{\vlambda}}) \right] 
= \nabla_{\text{\vmu}} \myexpect_{\pi_{\text{\vlambda}(\text{\vmu})}}[\bar\ell(\vtheta)] \Big|_{\text{\vmu}= \text{\vmu}(\text{\vlambda})} + \vlambda.
\]
This proves \cref{eq:fy_blr_gradient_identity}. Substituting this identity into the canonical-gradient step gives
\[
\vlambda_{t+1} = \vlambda_t - \rho_t \left[ \nabla_{\text{\vmu}} \myexpect_{\pi_{\text{\vlambda}(\text{\vmu})}}[\bar\ell(\vtheta)] \Big|_{\text{\vmu}=\text{\vmu}(\text{\vlambda}_t)} + \vlambda_t \right],
\]
which rearranges to
\[
\vlambda_{t+1} = (1-\rho_t)\vlambda_t - \rho_t \nabla_{\text{\vmu}} \myexpect_{\pi_{\text{\vlambda}(\text{\vmu})}}[\bar\ell(\vtheta)] \Big|_{\text{\vmu}= \text{\vmu}( \text{\vlambda}_t)} .
\]
The fixed-point condition follows by setting
\(\vlambda_{t+1}=\vlambda_t=\vlambda^*\) with \(\rho_t>0\).
\end{proof}


\section{$\postdeformidx$-Gaussians: $\postdeformidx$-exponential Family Form to Location-Scale}\label{app:q_gauss}

\subsection{Canonical $\postdeformidx$-exponential form}

We start from the $\postdeformidx$-exponential family canonical parametrization of $\postdeformidx$-Gaussians (\cref{eq:deformed_exp_families}). To derive \cref{alg:qvon} and \cref{alg:tvon} we use the following sufficient statistics and canonical parameters, which along with the deformation index $\postdeformidx$ determine the $\postdeformidx$-Gaussian density
\begin{align}
\vT(\vparam)=
\begin{pmatrix}
\vtheta\\ 
\vtheta\vtheta^\top
\end{pmatrix},
\qquad
\vlambda=
\begin{pmatrix}
\vS\vm\\ 
-\tfrac12\vS
\end{pmatrix}.
\end{align}
Both bounded-support ($\postdeformidx<1$), and heavy-tailed densities ($\postdeformidx>1$) can be written with subtractive normalization in $\postdeformidx$-exponential form
\begin{equation} \label{eq:q_canonical}
\varpost(\vtheta) =\exp_{\postdeformidx} \left(\vlambda^\top \vT(\vparam) - A_{2-\postdeformidx}(\vlambda)\right),
\qquad
\vlambda^\top \vT(\vparam)
= -\tfrac12(\vtheta-\vm)^\top\vS(\vtheta-\vm)+\tfrac12\vm^\top\vS\vm,
\end{equation}
with constants absorbed into the deformed cumulant $A_{2-\postdeformidx}$.

When $\boldsymbol{\postdeformidx<1}$ 
\begin{align} \label{eq:cumulant_bounded}
A_{2-\postdeformidx}(\vlambda) =
\frac{1}{2} \vm^\top \vS \vm - \frac{1}{2}
\left( R(\postdeformidx,D)^2 |\vS|^{1/\gamma} - \frac{2}{1-q}\right)
\end{align}
with constant $\gamma = D+\frac{2}{1-\postdeformidx}$, and where the max support radius $R$ has a closed form that we derive below. 

When $\boldsymbol{\postdeformidx>1}$ 
\begin{equation} \label{eq:cumulant_t}
A_{2-\postdeformidx}(\vlambda) =
\frac{1}{2} \vm^\top \vS \vm + \frac{1}{\postdeformidx-1}
\left( |\vS|^{-1/\nu} \frac{\nu}{\nu+D}\kappa_t(\nu, D) -1 \right)
\end{equation}
where $\Gamma$ is the Gamma function, and $\nu = \frac{2}{q-1} - D$ is the Student-$t$ parameter as we show in more detail below where we also define constant $\kappa_t$.

\subsection{Radial laws and location-scale form}
To derive the closed forms of the covariances of $\vtheta$ that we use to derive the Newton-like updates in \cref{ssec:qvon} for the different ranges of $\postdeformidx$ we need the following polar coordinate stochastic reparametrization. For feature dimension $D$ we have
\begin{align} \label{eq:polar}
r^2(\vtheta)\ =\ (\vtheta-\vm)^\top \vSigma^{-1}(\vtheta-\vm),
\qquad
\vtheta = \vm + \vSigma^{1/2}\vz,\quad \vz = \vSigma^{-1/2}(\vtheta-\vm) = r\vu,
\end{align}
where $\vu \sim \mathrm{Unif}(S^{D-1})$ on the unit sphere surface $S^{D-1} = \{ \vu: \vu^\top \vu = 1 \}$, and $r\ge 0$ is independent of $\vu$.

When $\boldsymbol{\postdeformidx<1}$ we use the determinant-rescaled elliptical scale matrix $\vSigma = |\vS|^{1/\gamma}\vS^{-1}$ to write $\varpost(\vtheta)$ in elliptical Pearson Type II, location-scale form, which is obtained from \cref{eq:q_canonical} and \cref{eq:cumulant_bounded} by the definition of $\exp_{\postdeformidx}$ (\cref{eq:defomed_exp}) and some algebraic manipulation 
\begin{align}
\varpost(\vtheta) = |\vSigma|^{-1/2} \left[\frac{1-\postdeformidx}{2}\left(R^2-r^2(\vtheta)\right)\right]_+^{\frac{1}{1-\postdeformidx}},
\qquad \mathrm{supp}(\varpost) = \{\vtheta: r^2(\vtheta)<R^2\}. 
\end{align}
Evaluating the integral $\int_{\text{\real}^D} \left[\frac{1-\postdeformidx}{2}\left(R^2-\| \vz \|^2 \right)\right]_+^{\frac{1}{1-\postdeformidx}} d\vz = 1$ in terms of the polar coordinate reparametrization from \cref{eq:polar} yields the closed form of the max support radius $R(\postdeformidx, D)$
\begin{equation}
R(\postdeformidx,D)= \left[ \frac{\Gamma \left(\frac{D}{2}+\frac{2-\postdeformidx}{1-\postdeformidx}\right)}{\pi^{D/2}\Gamma \left(\frac{2-\postdeformidx}{1-\postdeformidx}\right)}
\left(\frac{2}{1-\postdeformidx}\right)^{\frac{1}{1-\postdeformidx}}
\right]^{\frac{1-\postdeformidx}{2+D(1-\postdeformidx)}},
\end{equation}
where $\Gamma$ is the Gamma function.

The squared ``radial fraction'' follows a Beta distribution:
\begin{align}\label{eq:radial_bounded}
\frac{r^2}{R^2(\postdeformidx,D)}
\sim \mathrm{Beta} \left(\frac{D}{2}, \frac{2-\postdeformidx}{1-\postdeformidx}\right),
\qquad r\in[0,R(\postdeformidx,D)].
\end{align}
Hence \textbf{sampling} can be implemented as $\vu\sim\mathrm{Unif}(S^{D-1})$, $\frac{r^2}{R^2(\postdeformidx,D)} \sim\mathrm{Beta}(\frac{D}{2}, \frac{2-\postdeformidx}{1-\postdeformidx})$, $\vz \gets r \vu$, $\vtheta \gets \vm+\vSigma^{1/2} \vz$.

When $\boldsymbol{\postdeformidx>1}$  
\begin{align}
&\nu  = \frac{2}{\postdeformidx-1}-D \quad \mathrm{and} \\
& \vSigma  = \kappa_t|\vS|^{-1/\nu}\vS^{-1}  \quad \mathrm{where} \quad \kappa_t(\nu,D) = \frac{\nu+D}{\nu}\left[(\nu+D)^{D/2}\pi^{D/2}\frac{\Gamma(\nu/2)}{\Gamma((\nu+D)/2)}\right]^{2/\nu},
\end{align}
From \cref{eq:q_canonical} and \cref{eq:cumulant_t} by the definition of $\exp_{\postdeformidx}$ (\cref{eq:defomed_exp}) and some algebraic manipulation we obtain the elliptical Pearson Type VII form of the density
\begin{align} \label{eq:pearsonVII_S_form}
\varpost(\vtheta) 
&= \left( |\vS|^{-1/\nu} \frac{\nu }{\nu + D}\kappa_t +\frac{\postdeformidx-1}{2}(\vtheta-\vm)^\top\vS(\vtheta-\vm) \right)^{-\frac{1}{\postdeformidx-1}}
\propto \left(1 + \frac{1}{\nu} r^2(\vtheta) \right)^{-(\nu+D)/2}.
\end{align}

Finally, we can bring the density in Student-$t$ form: 
\begin{align}
\mathcal{T}_{\nu}(\vtheta\mid \vm,\vSigma) = \frac{1}{Z_{\nu}|\vSigma|^{1/2}} \left(1+\tfrac{1}{\nu}r^2(\vtheta)\right)^{-(\nu+D)/2},
\qquad
Z_\nu=\frac{\Gamma(\nu/2)(\nu\pi)^{D/2}}{\Gamma((\nu+D)/2)},
\end{align}
since with $\Sigma$ as defined above, \cref{eq:pearsonVII_S_form} is proportional to the standard Student-t kernel.  

In standardized coordinates the Pearson Type VII density is:
\begin{align}
\pi(\vz) =\frac{1}{Z_\nu}\left(1+\tfrac{1}{\nu}\|\vz\|^2\right)^{-(\nu+D)/2},
\qquad \vz\in\mathbb{R}^D .
\end{align}
Using again $\vz=r\vu$ with Jacobian $d\vz=r^{D-1} dr d\vu$ and integrating out the angular part $\vu$, the induced density of the radius $r$ is proportional to
\begin{align}
r^{D-1}\left(1+\tfrac{1}{\nu}r^2\right)^{-(\nu+D)/2},
\qquad r\in[0,\infty).
\end{align}
Equivalently, by the change of variables $r \mapsto r^2$ (i.e.\ $dr=\tfrac{1}{2}(r^2)^{-1/2} d(r^2)$), the induced
density of $r^2$ is proportional to
\begin{align}
(r^2)^{\frac{D}{2}-1}\left(1+\tfrac{1}{\nu}r^2\right)^{-(\nu+D)/2},
\qquad r^2\in(0,\infty),
\end{align}
which is a Beta-prime kernel. Hence,
\begin{align}
\frac{r^2}{\nu} \sim\ \mathrm{BetaPrime} \left(\frac{D}{2},\frac{\nu}{2}\right).
\end{align}
A standard equivalent representation is the chi-square ratio or Fisher-$F$ form:
\begin{align} \label{eq:radial_t}
r^2 \stackrel{d}{=}
\nu \frac{\chi^2_D}{\chi^2_{\nu}} = D F_{D,\nu},
\qquad r\in[0,\infty),
\end{align}
with independent chi-squares in the first expression.

We can \textbf{sample} $\vu\sim\mathrm{Unif}(S^{D-1})$, $U\sim\chi^2_D$, $V\sim\chi^2_\nu$ independently, set
$r\gets\sqrt{\nu\,U/V}$, $\vz\gets r\,\vu$, and finally $\vtheta\gets \vm+\vSigma^{1/2}\vz$. Equivalently, sample $F\sim F_{D,\nu}$ and set $r\gets\sqrt{D F}$.

\subsection{Expectation parameters}\label{ssec:expectation_params}
To derive the general mixed Bregman form of the BLR we use linear moments under $\varpost$:
\begin{align}
\vmu =\myexpect_{\varpost}[\vtheta]=\vm,
\qquad
\vM =\myexpect_{\varpost}[\vtheta\vtheta^\top]
=\vm\vm^\top+\Var_{\varpost}(\vtheta).
\end{align}

For any elliptical $\varpost$ with stochastic representation $\vtheta=\vm+\vSigma^{1/2}(r\vu)$, using $\vtheta-\vmu = \vSigma^{1/2}(r\vu)$
\begin{align} \label{eq:elliptical_var}
\Var_{\varpost}(\vtheta)
= \vSigma^{1/2}\myexpect[r^2]\myexpect[\vu\vu^\top]\vSigma^{1/2}
= \vSigma^{1/2}\myexpect[r^2] \frac{1}{D}\vI\vSigma^{1/2}
= \frac{\myexpect[r^2]}{D}\vSigma.
\end{align}

In particular, from \cref{eq:elliptical_var} and the radial laws \cref{eq:radial_bounded} and \cref{eq:radial_t}:
\begin{align}
\postdeformidx<1:\ \ \Var_{\varpost}(\vtheta)=\kappa(\postdeformidx,D) \vSigma,
\qquad
\postdeformidx>1:\ \ \Var_{\varpost}(\vtheta)=\frac{\nu}{\nu-2} \vSigma\quad(\nu>2)
\end{align}
with $\kappa = \frac{R^2}{D + \frac{2(2-\postdeformidx)}{1-\postdeformidx}}$ as in \cref{alg:qvon}.

\paragraph{Escort reweightings. }
The Price- and Stein-type identities that can be derived for $\postdeformidx$-Gaussians naturally introduce a $(2-\postdeformidx)$-escorts (see  App.~\ref{app:bonnet_price_stein} for more details on related bibliography, and a proof of the heavy-tailed Price-type identity):
\begin{align}
\tilde \pi^{2-\postdeformidx}(\vtheta) \propto \pi(\vtheta)  w_{2-\postdeformidx}(\vtheta),
\end{align}
with
\begin{align}
\postdeformidx<1:\quad w_{2-\postdeformidx}(\vtheta)=R^2(\postdeformidx,D)-r^2(\vtheta),
\qquad
\postdeformidx>1:\quad w_{2-\postdeformidx}(\vtheta)=1+\frac{1}{\nu}r^2(\vtheta). 
\end{align}

\subsection{FY Mixed-space Bregman between two heavy-tailed $q$-Gaussians}

Analogously to the expression in \citep[Prop. 18]{martins2022sparse} we derive the FY mixed-space Bregman divergence for the $q>1$ case. For $\vT(\vparam) = (\vtheta, \vtheta\vtheta^\top)$ and $\vlambda=(\vS \vm,-\tfrac12 \vS)$, we consider the quadratic form of a score $\eta_0(\vtheta) = \vlambda_0^\top \vT(\vparam)$ in the same linear family $\vlambda_0 = \begin{pmatrix} \vS_0 \vm_0 \\ -\tfrac12 \vS_0 \end{pmatrix}$. Up to an additive constant independent of $\vtheta$ $\eta_0(\vtheta) = -\tfrac12(\vtheta-\vm_0)^\top\vS_0(\vtheta-\vm_0)$. Let $\pi_0$ denote the $\postdeformidx$-Gaussian induced by this score, i.e.
\begin{align}
\pi_0(\vtheta) = \mathcal{T}_{\nu}(\vtheta \mid \vm_0, \vSigma_0),
\qquad
\vSigma_0=\kappa_t |\vS_0|^{-1/\nu}\vS_0^{-1}.
\end{align}

The FY mixed-space Bregman divergence induced by $\Omega_{2-\postdeformidx}$ is
\begin{align}
\mathrm{D}_{\Omega_{2-\postdeformidx}}(\eta_0, \pi) &= \Omega_{2-\postdeformidx}(\pi)-\Omega_{2-\postdeformidx}(\pi_0) -\vlambda_0^\top\left(\vmu(\pi)-\vmu(\pi_0)\right),
\end{align}
where $\vmu(\pi)=(\myexpect_{\pi}[\vtheta],\myexpect_\pi[\vtheta\vtheta^\top])$. Using $\myexpect_\pi[\vtheta] = \vm$ and $\myexpect_\pi [\vtheta\vtheta^\top] = \vm\vm^\top + \Var_\pi(\vtheta)$, we get 
\begin{align}
-\vlambda_0^\top(\vmu(\pi) - \vmu(\pi_0)) = \frac12(\vm-\vm_0)^\top\vS_0(\vm-\vm_0) \nonumber \\
+ \frac12 \mathrm{Tr} \left(\vS_0\left(\Var_{\pi}(\vtheta)-\Var_{\pi_0}(\vtheta)\right)\right).
\end{align}

For Student-$t$ with $\nu>2$, $\Var_\pi(\vtheta)=\frac{\nu}{\nu-2}\vSigma$, hence
\begin{align}
\frac12 \mathrm{Tr} \left(\vS_0\left(\Var_\pi(\vtheta)-\Var_{\pi_0}(\vtheta)\right)\right)
= \frac{\kappa_t\nu}{2(\nu-2)} \left( |\vS|^{-1/\nu} \mathrm{Tr}(\vS_0\vS^{-1}) - D |\vS_0|^{-1/\nu} \right).
\end{align}
Using the deformed cumulant expression in \cref{eq:cumulant_t} and
$\Omega_{2-\postdeformidx}(\pi)=\frac{1}{2-\postdeformidx}(\vlambda^\top\vmu-A_{2-\postdeformidx}(\vlambda))$, we obtain (up to a constant depending only on $(\nu,D)$)
\begin{align}
\Omega_{2-\postdeformidx}(\pi) = \text{const.}(\nu, D) -\frac{\kappa_t\nu(\nu+D)}{2(\nu-2)} |\vS|^{-1/\nu},
\end{align}
and the difference simplifies to
\begin{align}
\Omega_{2-\postdeformidx}(\pi)-\Omega_{2-\postdeformidx}(\pi_0)
= -\frac{\kappa_t\nu(\nu+D)}{2(\nu-2)} \left(|\vS|^{-1/\nu}-|\vS_0|^{-1/\nu}\right).
\label{eq:fy_t_entropy_diff}
\end{align}

Putting everything together we get the heavy-tail, $q>1$ FY mixed-space Bregman divergence
\begin{align}
\mathrm{D}_{\Omega_{2-\postdeformidx}}(\eta_0, \pi)
&= \frac12(\vm-\vm_0)^\top\vS_0(\vm-\vm_0)  \nonumber \\
&+ \frac{\kappa_t \nu}{2(\nu-2)}
\left[|\vS|^{-1/\nu}\left(\mathrm{Tr}(\vS_0 \vS^{-1})-(\nu+D)\right)
+\nu |\vS_0|^{-1/\nu} \right].
\end{align}

\section{Bonnet-, Price-, and Stein-type identities for $\postdeformidx$-Gaussians} \label{app:bonnet_price_stein}

In \cref{alg:qvon} and \cref{alg:tvon} we replace the Gaussian samples of VON with elliptical $\postdeformidx$‑Gaussian samples, using generalized Bonnet and Price identities related to appropriate elliptical Stein-type identities. The resulting $\postdeformidx$‑VON algorithms retain VON's form, differing by factors depending only on $(\postdeformidx, D)$ and closed-form escort reweighting derived from the dual; as $\postdeformidx \to 1$ we recover standard VON. 

By generalizing the VON sampling step from a Gaussian to a $\postdeformidx$-Gaussian distribution we obtain a more general class of scalable $\postdeformidx$-VON algorithms, which extend the BLR to optimize over a much wider class of distributions. We focus on variants of VON involving deformed exponential families with bounded, parameter-dependent support ($q < 1$), and with heavy tails ($q > 1$). Elliptical families include $\postdeformidx$‑Gaussians, which are encompassed under Pearson Type II (bounded support) and Type VII (heavy‑tailed) distributions for appropriate ranges of $\postdeformidx$, and provide bounded‑support and heavy‑tailed behaviors with tractable geometry and reparameterizations. Further, they admit Stein‑type identities which lead to generalized Bonnet and Price theorems \citep{landsman2008stein,landsman2015some,matsuzoe2011geometry}. The references concerning the heavy-tailed cases prove Stein-type identities but do not apply them to obtain Bonnet- and Price-type theorems, which we prove below for reference. Related proofs can be found in \citet{khan2017conjugate, lukashchukquotient}. A bounded-support $q$-Gaussian Stein identity is derived and used to prove Bonnet- and Price-type identities in \citet{qsteincpal}.

Here we review the gradient identities used to derive the $\postdeformidx$-VON updates (\cref{ssec:qvon}, \cref{alg:qvon} and \cref{alg:tvon}). Throughout, $\vtheta\in\real^D$ is the random variable and $\varpost(\vtheta)$ is the variational posterior with canonical parameters $\vlambda = \left(\vS \vm, -\tfrac12\vS \right)$. We write the same distribution in elliptical location-scale form $\varpost(\vtheta)=q_{\postdeformidx}(\vtheta\mid \vm,\vSigma)$, with $r^2(\vtheta) = (\vtheta-\vm)^\top \vSigma^{-1}(\vtheta-\vm)$.

\subsection{Bonnet-type identities} \label{app:bonnet}

For both $\postdeformidx<1$ and $\postdeformidx>1$, for differentiable $f$ with suitable integrability, the gradient with respect to the location parameter satisfies the Bonnet identity
\begin{align}
\nabla_{\text{\vm}} \myexpect_{\varpost} \left[f(\vtheta)\right]
= \myexpect_{\varpost}\left[\nabla_{\text{\vtheta}} f(\vtheta)\right].
\label{eq:bonnet_qgauss}
\end{align}

\subsubsection{Proof of the $q>1$ or Student-$t$ Bonnet-type identity}

For a proof of the $q<1$ case see \citet{qsteincpal}. 

Let $\varpost(\vtheta)=\mathcal{T}_{\nu}(\vtheta\mid \vm,\vSigma)$ with $\nu>2$ and support $\mathbb{R}^D$, 
\begin{align}
\varpost(\vtheta) = \frac{1}{Z_{\nu} |\vSigma|^{1/2}} \left(1+\tfrac{1}{\nu}r^2(\vtheta)\right)^{-(\nu+D)/2}.
\end{align}
Assume $f:\real^D\to\mathbb{R}$ is differentiable and integrable under $\varpost$, and that boundary terms vanish (e.g. $f$ and $\nabla_{\text{\vtheta}}f$ are bounded; it suffices that $\varpost(\vtheta)$ decays polynomially).

We show \cref{eq:bonnet_qgauss} holds componentwise. Fix $i\in\{1,\dots,D\}$. By dominated convergence we may differentiate under the integral:
\begin{align}
\frac{\partial}{\partial \vm_i} \myexpect_{\varpost}[f(\vtheta)] 
&= \int_{\mathbb{R}^D} f(\vtheta)\,\frac{\partial}{\partial \vm_i}\varpost(\vtheta)\,d\vtheta.
\end{align}
Since $\varpost(\vtheta)$ depends on $(\vtheta,\vm)$ only through $(\vtheta-\vm)$, we have the identity
\begin{align}
\frac{\partial}{\partial \vm_i}\varpost(\vtheta) = -\frac{\partial}{\partial \vtheta_i}\varpost(\vtheta).
\end{align}
Therefore,
\begin{align}
\frac{\partial}{\partial \vm_i}\,\myexpect_{\varpost}[f(\vtheta)]
&= -\int f(\vtheta) \frac{\partial}{\partial \vtheta_i}\varpost(\vtheta) d\vtheta
= -\int \frac{\partial}{\partial \vtheta_i} \left(\varpost(\vtheta) f(\vtheta)\right) d\vtheta
+\int \varpost(\vtheta) \frac{\partial f(\vtheta)}{\partial \vtheta_i} d\vtheta.
\end{align}
The total-derivative term vanishes by the boundary assumption, and thus
\begin{align}
\frac{\partial}{\partial \vm_i} \myexpect_{\varpost}[f(\vtheta)]
= \myexpect_{\varpost} \left[\frac{\partial f(\vtheta)}{\partial \vtheta_i}\right].
\end{align}
Stacking coordinates gives the vector form \cref{eq:bonnet_qgauss}: $\nabla_{\text{\vm}}\myexpect_{\varpost}[f(\vtheta)] = \myexpect_{\varpost}[\nabla_{\text{\vtheta}}f(\vtheta)]$.

\subsection{Price-type identities} \label{app:price}

Let $f$ be twice differentiable with $\nabla_{\text{\vtheta}} f$ bounded and $\nabla_{\text{\vtheta}}^2 f$ integrable. Then the gradient with respect to the scale parameter admits a generalized Price identity of the form
\begin{align}
\nabla_{\text{\vSigma}} \myexpect_{\varpost} \left[f(\vtheta)\right]
= \frac12 \myexpect_{\varpost} \left[\tau_{\postdeformidx}(\vtheta) \nabla_{\text{\vtheta}}^2 f(\vtheta)\right],
\label{eq:price_tau_form}
\end{align}
where the scalar factor $\tau_{\postdeformidx}(\vtheta)$ depends on the value of $\postdeformidx$:
\begin{align}
&\tau_{\postdeformidx}(\vtheta)=
\begin{cases}
\displaystyle
\frac{1-\postdeformidx}{2(2-\postdeformidx)}w_{2-q},
& \postdeformidx<1,\\
\displaystyle
\frac{\nu}{\nu+D-2}w_{2-q} =
\frac{\nu+r^2(\vtheta)}{\nu+D-2},
& \postdeformidx>1,  
\end{cases} 
\quad \text{where } \\
&w_{2-\postdeformidx}(\vtheta) =
\begin{cases}
R^2(\postdeformidx,D)-r^2(\vtheta), & \postdeformidx<1,\\
1+\tfrac{1}{\nu}r^2(\vtheta), & \postdeformidx>1.
\end{cases}
\end{align}

\paragraph{Escort form. } 
For both ranges of $q$
\begin{align}
\tilde \pi^{2-\postdeformidx}(\vtheta) 
\propto\ \varpost(\vtheta)w_{2-\postdeformidx}(\vtheta).
\end{align}
This means that \cref{eq:price_tau_form} can be rewritten compactly in a form that resembles the Gaussian Stein identity as
\begin{equation}
\nabla_{\text{\vSigma}} \myexpect_{\varpost} \left[f(\vtheta)\right] =
\frac{c_{\postdeformidx}}{2} \myexpect_{\tilde \pi^{2-\postdeformidx}} \left[\nabla_{\text{\vtheta}}^2 f(\vtheta)\right],
\label{eq:price_escort_form}
\end{equation}
with the constant
\begin{align}
c_{\postdeformidx} = \myexpect_\pi \tau_q (\vtheta) = 
\begin{cases}
\kappa(\postdeformidx,D), & \postdeformidx<1,\\[0.25em]
\dfrac{\nu}{\nu-2}, & \postdeformidx>1\ (\nu>2).
\end{cases}
\end{align}
When $\postdeformidx\to 1$ we retrieve the Gaussian identity $\nabla_{\text{\vSigma}} \myexpect_{\text{\gauss}(\text{\vm}, \text{\vSigma})}[f] = \frac12 \myexpect[\nabla_{\text{\vtheta}}^2 f]$.

\subsubsection{Diagonal Stein estimator}\label{app:stein_score}

Eq.~\eqref{eq:price_escort_form} shows that the gradient with respect to the scale parameter depends on the Hessian of $f$ averaged under the complementary escort $\tilde \pi^{2-\postdeformidx}$.
In our $\postdeformidx$-VON algorithms we restrict to diagonal $\vSigma$ and only require the diagonal of $\myexpect_{\tilde \pi^{2-\postdeformidx}}[\nabla_{\text{\vtheta}}^2 f(\vtheta)]$.
To avoid explicit Hessian evaluations, we use a Stein identity to rewrite these escort Hessian
diagonal entries in terms of first derivatives evaluated at samples $\vtheta \sim \varpost$.

Let $f:\real^D\to\mathbb{R}$ be differentiable and integrable under $\varpost$. Stein-type identities take the following score form
\begin{align}
\myexpect_{\varpost} \left[\nabla_{\text{\vtheta}} f(\vtheta)\right]
= -\myexpect_{\varpost} \left[f(\vtheta) \nabla_{\text{\vtheta}}\log \varpost(\vtheta)\right],
\label{eq:stein_score_general}
\end{align}
provided boundary terms vanish. For $\postdeformidx>1$ (unbounded support) a Stein-type identity has been proven by \citet{landsman2015some}. For $\postdeformidx<1$ the support is bounded but $\varpost(\vtheta)$ vanishes at the boundary, so a similar identity holds under mild regularity conditions on $f$ as shown in \cite{qsteincpal}.

For $\postdeformidx$-Gaussian densities with scale matrix $\vSigma$ the scores can be written explicitly as
\begin{align}
\nabla_{\text{\vtheta}}\log \varpost(\vtheta) = - \frac{1}{(2-q) \tau_q (\vtheta)} \vSigma^{-1}(\vtheta-\vm)
\end{align}

Equivalently, and more conveniently for our purposes, these score functions imply an escort covariance form
Stein identity: 
\begin{align}\label{eq:stein_escort_cov_form}
\myexpect_{\varpost} \left[(\vtheta-\vm) f(\vtheta)\right] = c_{\postdeformidx} \vSigma
\myexpect_{\tilde \pi^{2-\postdeformidx}} \left[\nabla_{\text{\vtheta}} f(\vtheta)\right].
\end{align}

Algs.~\ref{alg:qvon}-\ref{alg:tvon} only require an estimate of a diagonal $\myexpect_{\tilde \pi^{2-\postdeformidx}}[\nabla_{\text{\vtheta}}^2 f(\vtheta)]$. Assume $\vSigma=\mathrm{diag}(\vsigma^2)$ and define $\widehat\vg=\nabla_{\text{\vtheta}}f(\vtheta)$.
Then, componentwise,
\begin{align}
\myexpect_{\tilde \pi^{2-\postdeformidx}} \left[\frac{\partial^2 f(\vtheta)}{\partial \vtheta_i^2}\right]
= \frac{1}{c_{\postdeformidx}} \myexpect_{\varpost} \left[\frac{\partial f(\vtheta)}{\partial \vtheta_i} \frac{\vtheta_i-\vm_i}{\vsigma_i^2} \right],
\qquad i=1,\dots,D . 
\label{eq:diag_stein_estimator}
\end{align}
Thus a single-sample unbiased estimator for the diagonal Hessian under the complementary escort is
\begin{align}
\widehat \vh = \frac{1}{c_{\postdeformidx}} \widehat \vg \cdot (\vtheta-\vm) / \vsigma^2,
\quad \vtheta\sim \varpost,
\end{align}
which is exactly the form used in Algs.~\ref{alg:qvon} and ~\ref{alg:tvon}.

\subsubsection{Proof of the $q>1$ Student-$t$ Price-type identity }

For a proof of the $q<1$ Price-type identity see \cite{qsteincpal}. 

Let $\varpost(\vtheta)=\mathcal{T}_{\nu}(\vtheta\mid \vm,\vSigma)$ with $\nu>2$.
Assume $f$ is twice continuously differentiable, integrable under $\varpost$, and that boundary terms vanish for the integrations by parts below (it is sufficient that $f$ and $\nabla_{\text{\vtheta}}f$ are bounded).

We prove the componentwise identity
\begin{align}\label{eq:price_studentt_component_proof}
\frac{\partial}{\partial \vSigma_{ij}}\,\myexpect_{\varpost}[f(\vtheta)]
= \frac{1}{2}\myexpect_{\varpost}\!\left[
\frac{\nu+r^2(\vtheta)}{\nu+D-2}\;
\frac{\partial^2 f(\vtheta)}{\partial \vtheta_i\,\partial \vtheta_j}
\right],
\qquad i,j\in\{1,\dots,D\},
\end{align}
which is equivalent to \cref{eq:price_tau_form} for $\postdeformidx>1$ since $\frac{\nu+r^2}{\nu+D-2} = \frac{\nu}{\nu+D-2} \left(1+\frac{1}{\nu}r^2\right)$.

\paragraph{Derivative with respect to $\vSigma$ of the expectation. } By dominated convergence,
\begin{align}
\frac{\partial}{\partial \text{\vSigma}_{ij}} \myexpect_{\varpost}[f(\vtheta)]
&= \int f(\vtheta) \frac{\partial}{\partial \text{\vSigma}_{ij}}\varpost(\vtheta) d\vtheta
= \myexpect_{\varpost} \left[f(\vtheta) \frac{\partial}{\partial \text{\vSigma}_{ij}} \log \varpost(\vtheta)\right].
\end{align}
Writing $ \log \varpost(\vtheta)=\text{const.} - \tfrac12 \log |\vSigma| - \tfrac{\nu+D}{2} \log \left(1+\tfrac{1}{\nu}r^2(\vtheta)\right)$, and using $\frac{\partial}{\partial \text{\vSigma}_{ij}} \log |\vSigma| = (\vSigma^{-1})_{ji}$
and $\frac{\partial}{\partial \vSigma_{ij}}r^2(\vtheta) = -(\vSigma^{-1}(\vtheta-\vm)(\vtheta-\vm)^\top\vSigma^{-1})_{ij}$, we obtain
\begin{align}\label{eq:studentt_dlogq_dSigma}
\frac{\partial}{\partial \vSigma_{ij}}\log \varpost(\vtheta)
= -\frac{1}{2}(\vSigma^{-1})_{ij} +\frac{\nu+D}{2(\nu+r^2(\vtheta))} \left(\vSigma^{-1}(\vtheta-\vm)(\vtheta-\vm)^\top\vSigma^{-1}\right)_{ij}.
\end{align}
Define $\va(\vtheta)=\vSigma^{-1}(\vtheta-\vm)$ so that $r^2(\vtheta) = (\vtheta-\vm)^\top\va(\vtheta)$ and $\left(\vSigma^{-1}(\vtheta-\vm)(\vtheta-\vm)^\top\vSigma^{-1}\right)_{ij} = \va_i(\vtheta) \va_j(\vtheta)$.
Then \cref{eq:studentt_dlogq_dSigma} becomes
\begin{align}\label{eq:studentt_dlogq_dSigma_a}
\frac{\partial}{\partial \vSigma_{ij}}\log \varpost(\vtheta)
= -\frac{1}{2}(\vSigma^{-1})_{ij} +\frac{\nu+D}{2(\nu+r^2(\vtheta))} \va_i(\vtheta)\va_j(\vtheta).
\end{align}
Therefore,
\begin{align}\label{eq:lhs_price_studentt}
\frac{\partial}{\partial \vSigma_{ij}} \myexpect_{\varpost}[f(\vtheta)]
= \myexpect_{\varpost} \left[f(\vtheta) \left(-\frac{1}{2}(\vSigma^{-1})_{ij} + \frac{\nu+D}{2(\nu+r^2(\vtheta))} \va_i(\vtheta)\va_j(\vtheta)\right) \right].
\end{align}

\paragraph{Integration by parts to express \cref{eq:lhs_price_studentt} in terms of the Hessian. } We use the Student-$t$ score:
\begin{align}\label{eq:studentt_score_used}
\frac{\partial}{\partial \vtheta_i}\log \varpost(\vtheta)
= -\frac{\nu+D}{\nu+r^2(\vtheta)} \va_i(\vtheta).
\end{align}
An important simplification is 
\begin{align}\label{eq:key_identity_q_weight}
\frac{\partial}{\partial \vtheta_i} \left(\varpost(\vtheta) (\nu+r^2(\vtheta))\right)
= -(\nu+D-2) \varpost(\vtheta) \va_i(\vtheta),
\end{align}
which follows from the product rule, \cref{eq:studentt_score_used}, and $\frac{\partial}{\partial \vtheta_i}r^2(\vtheta)=2\,\va_i(\vtheta)$.

Let $w(\vtheta) = \frac{\nu+r^2(\text{\vtheta})}{\nu+D-2}$. Then \cref{eq:key_identity_q_weight} is equivalent to
\begin{align}\label{eq:derivative_qw}
\frac{\partial}{\partial \vtheta_i}\left(\varpost(\vtheta) w(\vtheta)\right)
= -\varpost(\vtheta) \va_i(\vtheta).
\end{align}
Now integrate by parts twice:
\begin{align}
\myexpect_{\varpost} \left[w(\vtheta) \frac{\partial^2 f(\vtheta)}{\partial \vtheta_i\partial \vtheta_j}\right]
&= \int \varpost(\vtheta) w(\vtheta) \frac{\partial^2 f(\vtheta)}{\partial \vtheta_i\partial \vtheta_j} d\vtheta
= -\int \frac{\partial}{\partial \vtheta_i} \left(\varpost(\vtheta) w(\vtheta)\right) \frac{\partial f(\vtheta)}{\partial \vtheta_j} d\vtheta \nonumber\\
&\overset{\cref{eq:derivative_qw}}{=}
\int \varpost(\vtheta) \va_i(\vtheta) \frac{\partial f(\vtheta)}{\partial \vtheta_j} d\vtheta
= -\int \frac{\partial}{\partial \vtheta_j} \left(\varpost(\vtheta) \va_i(\vtheta)\right) f(\vtheta) d\vtheta.
\label{eq:ibp_two_steps}
\end{align}
We compute the derivative inside the last integral. Since $\frac{\partial}{\partial \text{\vtheta}_j} \va_i(\vtheta) = (\vSigma^{-1})_{ij}$ and $\frac{\partial}{\partial \text{\vtheta}_j} \varpost(\vtheta) = \varpost(\vtheta) \frac{\partial}{\partial \text{\vtheta}_j} \log \varpost(\vtheta)$, we get
\begin{align}
\frac{\partial}{\partial \vtheta_j} \left(\varpost(\vtheta) \va_i(\vtheta)\right)
&= \varpost(\vtheta) (\vSigma^{-1})_{ij} +\varpost(\vtheta) \va_i(\vtheta) \frac{\partial}{\partial \vtheta_j} \log\varpost(\vtheta) \nonumber\\
&\overset{\cref{eq:studentt_score_used}}{=}
\varpost(\vtheta) (\vSigma^{-1})_{ij} -\varpost(\vtheta) \va_i(\vtheta) \frac{\nu+D}{\nu+r^2(\vtheta)} \va_j(\vtheta).
\end{align}
Substituting into \cref{eq:ibp_two_steps} yields
\begin{align}\label{eq:hess_to_score_studentt}
\myexpect_{\varpost} \left[w(\vtheta) \frac{\partial^2 f(\vtheta)}{\partial \vtheta_i\partial \vtheta_j}\right]
= \myexpect_{\varpost} \left[f(\vtheta)\left(-(\vSigma^{-1})_{ij} + \frac{\nu+D}{\nu+r^2(\vtheta)} \va_i(\vtheta)\va_j(\vtheta) \right) \right].
\end{align}
Comparing \cref{eq:hess_to_score_studentt} with \cref{eq:lhs_price_studentt} gives exactly the Price-type identity \cref{eq:price_studentt_component_proof}:
\begin{align}
\frac{\partial}{\partial \vSigma_{ij}} \myexpect_{\varpost}[f(\vtheta)]
= \frac12 \myexpect_{\varpost} \left[w(\vtheta) \frac{\partial^2 f(\vtheta)}{\partial \vtheta_i\partial \vtheta_j}\right],
\end{align}
which completes the proof.

\section{Tsallis Negentropy and $q$-VON}\label{sec:qvon}

As the main example we replace the Shannon negentropy by the Tsallis negative entropy $\Omega_{2-\postdeformidx}$, which is maximized by $\postdeformidx$-exponential families (see App.~\cref{app:maxent} for a proof). When we use quadratic statistics the maxent family becomes a $\postdeformidx$-Gaussian, which has bounded support for $\postdeformidx<1$ and heavy tails for $\postdeformidx>1$. Respectively, these lead to the $\postdeformidx$-VON and Student-$t$-VON algorithms. Thus, the quadratic statistic remains the same, while the choice of $\Omega$ changes the support and tail behavior.
\begin{align}\label{eq:tsallis}
    \Omega_{2-\postdeformidx}(\pi) = \left\{
    \begin{array}{ll}
       \hspace{-0.15cm}\frac{\mathbb{E}_{\pi}[\pi^{1-q} - 1] }{\bigl( (2-q)^2-(2-q) \bigr)} & \text{if $q \ne 1$,} \\
       \hspace{-0.15cm}\mathbb{E}_{\pi}[\log \pi]
         & \text{if $q = 1$,} 
    \end{array}
    \right.
\end{align}
Using entropic index $2-\postdeformidx$, and linear expectation constraints, the maxent problem in \cref{eq:maxent_psi} is optimized by $\postdeformidx$-exponential families (see App.~\ref{ssec:tsallis_maxent} for a proof):\footnote{We assume a constant base measure for simplicity; what follows would work with general base measures.} 
\begin{align}\label{eq:deformed_exp_families}
    \varpost(\vtheta) = \exp_{\postdeformidx}(\vlambda^\top \vT(\vparam) - A_{2-\postdeformidx}(\vlambda)),
\end{align}
where $A_{2-\postdeformidx}(\vlambda)$ is the deformed cumulant (we give detailed closed form examples in App.~\ref{app:q_gauss}; see also \citet{naudts2011generalised, amari2011geometry, amari2016information, martins2022sparse}). 
The $\postdeformidx$-exp and $\postdeformidx$-log functions are defined as
\begin{align} \label{eq:defomed_exp}
    \exp_{\postdeformidx}(u) = \left\{
    \begin{array}{ll}
        [1 + (1-\postdeformidx) u]_+^{1/(1-\postdeformidx)} & \text{if $\postdeformidx \ne 1$} \\
        \exp u & \text{if $\postdeformidx=1$}
    \end{array}
    \right.,
    \qquad
    \log_{\postdeformidx}(u) = 
    \left\{
    \begin{array}{ll}
        \frac{u^{1-\postdeformidx}-1}{1-\postdeformidx} & \text{if $\postdeformidx \ne 1$} \\
        \log u & \text{if $\postdeformidx=1$}
    \end{array}
    \right.,
\end{align}
and exponential families are recovered when $\postdeformidx=1$.

We will repeatedly use escort distributions associated with a density $\pi$. For any order $a>0$ such that $\int \pi(\vtheta)^a d\vtheta < \infty$, we define the escort of order $a$ by
\[
\tilde{\pi}^{a}(\vtheta) = \frac{\pi(\vtheta)^a}{\int \pi(\vtheta)^a d\vtheta}.
\]
In the Tsallis setting below, the orders $a=\postdeformidx$ and $a=2-\postdeformidx$ arise naturally in the gradient identities and in the Newton-like updates.

When $q$-exponential families are derived by maximum Tsallis entropy subject to linear expectation constraints, $A_{2-\postdeformidx}$ is related to $\Omega$ and $\Omega^*$ as follows \citep[Prop.10, App. B]{martins2022sparse}:
\begin{align}\label{eq:deformed_exp_formulae}
    \Omega_{2-\postdeformidx}(\varpost) &= \frac{1}{2-\postdeformidx} (\vlambda^\top \vmu(\vlambda) - A_{2-\postdeformidx}(\vlambda))\nonumber\\
    \Omega_{2-\postdeformidx}^*(\eta_{\text{\vlambda}}) &= (1- \postdeformidx) \Omega_{2-\postdeformidx}(\varpost) + A_{2-\postdeformidx}(\vlambda)
    = \frac{1}{2-\postdeformidx} \left( (1-\postdeformidx) \vlambda^\top \vmu(\vlambda) + A_{2-\postdeformidx}(\vlambda)\right),
\end{align}
where $\vmu(\vlambda) = \myexpect_{\text{\vtheta} \sim \varpost}[\vT(\vparam)] = \nabla_{\text{\vlambda}} \Omega_{2-\postdeformidx}^*(\eta_{\text{\vlambda}})$, and
\begin{align}\label{eq:fisher}
    \frac{\partial \text{\vmu}(\text{\vlambda})}{\partial \text{\vlambda}} = \nabla^2_{\text{\vlambda}} \Omega^*_{2-\postdeformidx}(\eta_{\text{\vlambda}}) =: \vF_{2-\postdeformidx}(\vlambda). 
\end{align}
\cref{eq:fisher} is useful in invoking the chain rule to convert gradients into natural or canonical gradients and vice-versa, and the Hessian of $\Omega^*$ with respect to the canonical parameters $\vF_{2-\postdeformidx}$ recovers the Fisher matrix when $\postdeformidx=1$. 

From \cref{eq:deformed_exp_formulae} we obtain:
\begin{align*}
    \nabla_{\text{\vmu}} \Omega_{2-\postdeformidx}(\varpost) &= \nabla_{\text{\vmu}} \myexpect_{\varpost}\left[\frac{1}{2-\postdeformidx}\log_{\postdeformidx} \varpost(\vtheta)\right]\nonumber\\
    &= \frac{1}{2-\postdeformidx} \nabla_{\text{\vmu}} \int_{\text{\vtheta}} \varpost(\vtheta) [\vlambda^\top \vT(\vtheta) - A_{2-\postdeformidx}(\vlambda)]\nonumber\\
    &= \frac{1}{2-\postdeformidx} \nabla_{\text{\vmu}} [\vlambda^\top \underbrace{\myexpect_{\varpost}[\vT(\vtheta)]}_{\text{\vmu}(\text{\vlambda})} - A_{2-\postdeformidx}(\vlambda)] \nonumber\\
    &= \frac{1}{2-\postdeformidx} \left[ \frac{\partial \text{\vlambda}}{\partial \text{\vmu}} \vmu + \vlambda - \frac{\partial\vlambda}{\partial \vmu} \frac{\partial A_{2-\postdeformidx}(\vlambda)}{\partial \vlambda} \right]\nonumber\\
    &= \frac{1}{2-\postdeformidx} \left[\vlambda + \vF_{2-\postdeformidx}^{-1}(\vlambda) (\vmu - \tmu_{\postdeformidx}) \right].
\end{align*}
We define $\tmu_{\postdeformidx} = \nabla A_{2-\postdeformidx}(\vlambda) = \myexpect_{\tilde{\pi}^{\postdeformidx}} [\vT(\vparam) ] $ as the expectation with respect to the escort distribution  $\tilde{\pi}^{\postdeformidx}$ of order $\postdeformidx$, which is defined as $\tilde{\pi}^{\postdeformidx} = \pi^{\postdeformidx}/\int \pi^{\postdeformidx} d \text{\vtheta}$. \cref{eq:deformed_exp_formulae} further implies that
\begin{align}
    \vmu - \tilde{\vmu}_{\postdeformidx} = ( 1- \postdeformidx) \vF_{2-\postdeformidx} \vlambda,
\end{align}
and substituting above, it follows that when $\Omega$ is the Tsallis negentropy
\begin{align}\label{eq:grad_entropy}
    \nabla_{\text{\vmu}} \Omega_{2-\postdeformidx} (\varpost) = \frac{1}{2-\postdeformidx}(\vlambda + (1-\postdeformidx)\vlambda) = \vlambda, 
\end{align}
which we use below to derive Newton-like generalized BLR algorithms (\cref{alg:qvon}, \cref{alg:tvon}).

\subsection{$\postdeformidx$-Variational Online Newton}\label{ssec:qvon}

We now specialize the generalized BLR to the quadratic statistic $\vT(\vparam)^\top = (\vparam,\vparam\vparam^\top)$, $\vtheta \in \real^D$, which induces a location-scale family of $\postdeformidx$-Gaussians.
The derivation in this subsection proceeds in three steps. First, we write the generic canonical update for this quadratic family in terms of the expectation parameters $(\vmu,\vM)$. Second, we convert gradients with respect to $(\vmu,\vM)$ into gradients with respect to the location-scale parameters $(\vm,\vSigma)$. Third, we specialize our result to bounded-support $\postdeformidx<1$ and the heavy-tailed $\postdeformidx>1$ cases, where Bonnet- and Price-type identities yield the Newton-like updates implemented in \cref{alg:qvon} and \cref{alg:tvon}.
The derivation follows immediately below the final algorithms, which we present first for orientation.

First, we connect the pseudocode to the general canonical update in \cref{eq:non_blr}. 
In both $\postdeformidx<1$ and $\postdeformidx>1$ cases we start from the same canonical update, and what changes is the relation between expectation parameters and location-scale parameters, as well as the corresponding stochastic identities used to evaluate the gradients. In Variational Online Newton (VON) gradients and Hessians are estimated by sampling from a Gaussian variational posterior \cite{khan2018fast}. VON-style natural-gradient steps result in fast convergence and competitive test accuracy in large-scale deep learning applications \citep{shen2402variational}. Here, we generalize these updates to bounded-support (\cref{alg:qvon}) and heavy-tailed (\cref{alg:tvon}) $q$-Gaussian posteriors in a unified way. Remarkably, our generalizations require only small changes in the VON pseudo-code \citep[Alg.~1]{shen2402variational}, highlighted in blue.

Using the Tsallis negentropy $\Omega_{2-\postdeformidx}$ in \cref{eq:non_blr} and sufficient statistics $\vT(\vparam)^\top = (\vparam, \vparam\vparam^\top)$ we obtain $\postdeformidx$-Gaussian families with canonical parameters $\vnatparam^\top = (\vS\vm, -\half \vS)$.
$\postdeformidx$-Gaussians include both bounded-support and heavy-tailed (Student-$t$) subfamilies. The sufficient statistics and canonical parameters determine the elliptical geometry (that is, the quadratic form), but the value of $\postdeformidx$ determines the support and tail behavior. We derive the specific forms that \cref{eq:non_blr} takes for $\postdeformidx<1$ and $\postdeformidx>1$, resulting respectively in \cref{alg:qvon} and \cref{alg:tvon}, where in blue we highlight differences from the BLR's Variational Online Newton form \cite{khan2018fast} and IVON~\citep[Alg.~1]{shen2402variational}. Detailed background on $q$-Gaussians can be found in Appendix~\ref{app:q_gauss}.

We begin by writing \cref{eq:non_blr} more explicitly as 
\begin{align}\label{eq:q_gauss_lr}
    \vlambda_{t+1} =  
    (1-\rho_t) \vlambda_t - \rho_t 
    \begin{bmatrix}
    \nabla_{\text{\vmu}} \myexpect_{\varpost}   \bar \ell(\vtheta) \\ 
    \nabla_{\text{\vM}} \myexpect_{\varpost}   \bar \ell(\vtheta)
    \end{bmatrix}.
\end{align}
For both $q<1$ and $q>1$ the expectation parameters are
\begin{align} \label{eq:exp_param}
    \begin{bmatrix} 
    \vmu \\ 
    \vM   
    \end{bmatrix}  =
    \begin{bmatrix} 
    \myexpect_{\varpost} [\text{\vtheta} ] \\ 
    \myexpect_{\varpost} [\text{\vtheta} \text{\vtheta}^\top ]  
    \end{bmatrix}  
    = \begin{bmatrix} \vm \\
    \Var_{\varpost} (\vtheta) + \vm \vm^\top  
    \end{bmatrix}.
\end{align}
Importantly, $\Var (\vtheta)$ has distinct forms for $q<1$ and $q>1$, and it is \textit{not} equal to the canonical parameter $\vS^{-1}$ or the scale parameter $\vSigma$, unless $\postdeformidx=1$ in which case \smash{$\vS^{-1} = \vSigma = \Var(\vtheta)$}.

\paragraph{Chain rule. } To use Bonnet- and Price-type stochastic gradient identities, we need a chain rule to express gradients with respect to expectation parameters $(\vmu, \vM)$ as gradients with respect location-scale parameters $\vm(\vmu), \vSigma(\vmu, \vM)$. For any function $f(\vm(\vmu), \vSigma(\vmu, \vM))$, a straightforward generalization of the well known Gaussian chain rule \citep{opper2009variational, khan2017conjugate} is 
\begin{align}
    \nabla_{\text{\vmu}} f \left(\vm, \vSigma\right) 
     = \nabla_{\text{\vm}} f \nabla_{\text{\vmu}} \vm  + \nabla_{\text{\vSigma}} f  \nabla_{\text{\vmu}} \vSigma \label{eq:grad_m1}, \\
    \nabla_{\text{\vM}} f\left(\vm, \vSigma \right) 
     = \nabla_{\text{\vm}} f  \nabla_{\text{\vM}} \vm + \nabla_{\text{\vSigma}}  f  \nabla_{\text{\vM}} \vSigma \label{eq:grad_m2}. 
\end{align}

\paragraph{Bounded support. } When $\postdeformidx<1$, bounded-support $q$-Gaussians have maximum support radius $R(q, D)$. The max support radius has a closed-form (reviewed in App.~\ref{app:q_gauss}), which avoids dependence on $\vS^{-1}$ by expressing the scale matrix as a determinant reweighting \smash{$\vSigma = |\vS|^{ 1/\gamma} \vS^{-1}$} with constant \smash{$\gamma = D + \frac{2}{1-\postdeformidx}$}. From \cref{eq:exp_param} we can show $\Var_{\varpost} (\vtheta) =  \kappa \vSigma $ with constant \smash{$\kappa = R^2/(D + 2(2-\postdeformidx)/(1-\postdeformidx))$} (see App.~\ref{ssec:expectation_params}). 

From \cref{eq:grad_m1} and \smash{$ \partial \text{\vSigma}/\partial \text{\vmu} = -2 \kappa^{-1} \vm $}, 
it follows that
\begin{align}
\nabla_{\text{\vmu}} f = \nabla_{\text{\vm}} f - \frac{2}{\kappa}  \nabla_{\mathbf{\Sigma}} f  \vm, 
\end{align}
and from \cref{eq:grad_m2} \smash{$ \nabla_{\text{\vM}} \vm = 0 $} and \smash{$\partial\text{\vSigma}/\partial \text{\vM} = \kappa^{-1}$} that
\begin{align}
\nabla_{\text{\vM}} f = \kappa^{-1} \nabla_{\text{\vSigma}} f.
\end{align}
Setting $f = \bar \ell$, \cref{eq:q_gauss_lr} becomes
\begin{align} \label{eq:bounded_qvon}
   \vlambda_{t+1} &= (1-\rho_t) \vlambda_t 
   - \rho_t  
   \begin{bmatrix}
   \nabla_{\text{\vm}} \underset{\text{\vtheta}\sim \varpost}{\myexpect}  \bar \ell(\vtheta) - 2 \kappa^{-1} \left( \nabla_{\text{\vSigma}} \underset{\text{\vtheta}\sim \varpost}{\myexpect} [ \bar \ell(\vtheta) ]\right)  \vm  \\
   \kappa^{-1} \nabla_{\text{\vSigma}} \underset{\text{\vtheta}\sim \varpost}{\myexpect} [\bar \ell(\vtheta)]
   \end{bmatrix}.
\end{align}
Rearranging \cref{eq:bounded_qvon}, and using Price- and Bonnet-type identities for bounded-support $\postdeformidx$-Gaussians, to push the gradients inside the expectations~\citep{qsteincpal}, we obtain the bounded-support Newton-like generalized BLR (\cref{alg:qvon}): 
\begin{align}
    \vS_{ t+1} &= (1-\rho_t) \vS_{ t} + \rho_t \frac{1}{\kappa D}  \left( \myexpect_{\pi_t} [r^2(\vtheta) ] \myexpect_{\tilde{\pi}_t^{2-\postdeformidx}} [\nabla_{\text{\vtheta}}^2 \bar \ell (\vtheta) ] \right) 
    = (1-\rho_t) \vS_{ t} + \rho_t  \myexpect_{\tilde{\pi}_t^{2-\postdeformidx}} [\nabla_{\text{\vtheta}}^2 \bar \ell (\vtheta) ] \label{eq:bounded_s_update}\\
    \vm_{t+1}  &= \vm_t - \rho_t \vS_{t+1}^{-1} \myexpect_{\pi_t} [\nabla  \bar \ell(\vtheta)], \label{eq:bounded_m_update}
\end{align}
where $r^2(\vtheta) = (\vtheta-\vm)^\top\vSigma^{-1}(\vtheta-\vm)$ and $\tilde{\pi}^{2-\postdeformidx}$ is the escort distribution of order $2-\postdeformidx$ of $\varpost$. 
The final form of \cref{eq:bounded_s_update} follows by canceling the factor $\kappa^{-1}D$ with $ \myexpect_{\pi} [r^2(\vtheta)]$, and expressing the expectation with respect to the $(2-\postdeformidx)$-escort.


\paragraph{Heavy tails. } $\postdeformidx$-Gaussians with $\postdeformidx>1$ resulting from \cref{eq:deformed_exp_families} coincide with Student-$t$ densities with degrees of freedom \smash{$\nu=2/(\postdeformidx-1)-D$} and unbounded support $\real^D$ (further details are reviewed in App.~\ref{app:q_gauss}). We assume \smash{$\nu>2 \iff 1<\postdeformidx<(D+4)/ (D+2)$} so that the covariance exists (and may optionally require $\nu>4$ if we need finite kurtosis). We write the Student-$t$ scale matrix as $\vSigma = \kappa_t |\vS|^{-1/\nu} \vS^{-1}$ with constant 
\[
\kappa_t(\nu, D) = \frac{\nu+D}{\nu} \left[ (\nu+D)^{D/2}\pi^{D/2}\frac{\Gamma(\nu/2)}{\Gamma((\nu+D)/2)} \right]^{2/\nu},
\]
where $\Gamma$ is the Gamma function. Using the Student-$t$ parametrization \smash{$\varpost(\vtheta)=\mathcal{T}_\nu(\vtheta \mid \vm,\vSigma)$}, we can derive from \cref{eq:exp_param} that $\Var_{\varpost}(\vtheta)=\nu/(\nu-2) \vSigma $ (see App.~\ref{ssec:expectation_params}). 

Thus, from \cref{eq:grad_m1} we obtain
\begin{align}
\nabla_{\text{\vmu}} f = \nabla_{\text{\vm}} f - 2  \frac{\nu-2}{\nu} \nabla_{\mathbf{\Sigma}} f  \vm,
\end{align}
and, from \cref{eq:grad_m2}, since \smash{$\partial \text{\vSigma}/\partial \text{\vmu} =-2/\nu (\nu-2) \vm$} and \smash{$\partial \text{\vSigma}/\partial \text{\vM} = (\nu-2)/ \nu $} 
\begin{align}
     \nabla_{\text{\vM}} f =  \frac{\nu-2}{\nu} \nabla_{\text{\vSigma}} f.
\end{align}
Again, setting $f=\bar\ell$, the chain rule of \cref{eq:q_gauss_lr} results in 
\begin{align}
\vlambda_{t+1} \label{eq:studentt_lambda_update}
&=(1-\rho_t)\vlambda_t  \\ \nonumber
&-\rho_t
\begin{bmatrix}
\nabla_{\text{\vm}} \myexpect_{\pi}[\bar\ell(\vtheta)]
-2\frac{\nu-2}{\nu} \left(\nabla_{\text{\vSigma}}\myexpect_{\pi}[\bar\ell(\vtheta)] \right) \vm
\\
\frac{\nu-2}{\nu}\nabla_{\text{\vSigma}} \myexpect_{\pi}[\bar\ell(\vtheta)] 
\end{bmatrix}.
\end{align}
In App.~\ref{app:bonnet_price_stein} we prove the Student-$t$ Bonnet-type identity $\nabla_{\text{\vm}}\myexpect_{\pi}[\bar\ell(\vtheta)] = \myexpect_{\pi}[\nabla_{\text{\vtheta}} \bar\ell(\vtheta)]$ and the Student-$t$ Price-type identity:
\begin{align}\label{eq:studentt_price}
\nabla_{\text{\vSigma}}\myexpect_{\pi}[\bar\ell(\vtheta)] &= \frac{\nu}{2(\nu+D-2)} \myexpect_{\pi} \left[\left(1+\tfrac{1}{\nu}r^2(\vtheta)\right)\nabla_{\text{\vtheta}}^2\bar\ell(\vtheta)\right].
\end{align}
Similarly to the bounded-support case, we can express the Price-type identity in terms of the $(2-\postdeformidx)$-escort 
\begin{align}
    \nabla_{\text{\vSigma}} \myexpect_{\pi}[\bar\ell(\vtheta)]
    = \frac{\nu}{2 (\nu-2)} \myexpect_{\tilde \pi} \left[\nabla_{\text{\vtheta}}^2\bar\ell(\vtheta)\right],
\label{eq:studentt_hess}
\end{align}
which is a neat fact that, as far as we know, has not been noticed before. Since \smash{$\postdeformidx=1+2/(\nu+D)$} and \smash{$\pi(\vtheta)\propto\left(1+ r^2(\vtheta)/\nu\right)^{-(\nu+D)/2}$}, with $r^2$ defined as in the bounded-support case, we can express the \smash{$(2-\postdeformidx)$-escort} as a reweighting of the base distribution $\pi$:
\begin{align}
    \tilde \pi(\vtheta)^{2-\postdeformidx} = \frac{\pi(\text{\vtheta})^{2-\postdeformidx}}{\int \pi(\text{\vtheta})^{2-\postdeformidx} d \text{\vtheta}} 
    \propto \pi(\vtheta)\left(1+\tfrac{r^2}{\nu}(\vtheta)\right).
\end{align}
For $\nu>2$, 
\[
\myexpect_{\varpost}[r^2(\vtheta)] = \frac{D \nu}{\nu-2}\Rightarrow \myexpect_{\varpost}[1+ \frac{r^2}{\nu}]=\frac{D+\nu-2}{\nu-2},
\]
and substituting into \cref{eq:studentt_price} yields \cref{eq:studentt_hess}.

Thus, we obtain the Student-$t$ Newton-like canonical parameter updates that are used in \cref{alg:tvon} by solving \cref{eq:studentt_lambda_update}:
\begin{align}
    \vS_{t+1} = (1-\rho_t)\vS_{t} +\rho_t \myexpect_{\tilde{\pi}_t^{2-\postdeformidx}} \left[\nabla_{\text{\vtheta}}^2\bar\ell(\vtheta)\right], \label{eq:studentt_sigma_nat_update} \\
    \vm_{t+1} =\vm_t-\rho_t \vS_{t+1}^{-1} \myexpect_{\pi_t}\left[\nabla_{\text{\vtheta}}\bar\ell(\vtheta)\right].
\label{eq:studentt_m_update}
\end{align}

\paragraph{Implementation. } In both the bounded-support and heavy-tailed cases we exploit escort distributions to mirror the familiar Gaussian stochastic gradient expressions in an elegant way, and to highlight the probabilistic structure that would otherwise be obscured by repeated integration by parts. For efficient sampling in \cref{alg:qvon} and \cref{alg:tvon}, we express $\tilde{\pi}^{2-\postdeformidx}$ as a reweighting of $\pi_{\text{\vlambda}}$. When $\postdeformidx>1$ the $(2-\postdeformidx)$-escort relatively upweights larger values of $r(\vtheta)$ since typical samples dominate $\myexpect_\pi[\nabla\bar\ell]$ and the mean update is mode-driven. The scale update in this case is more tail-aware since atypical large-radius samples influence the Hessian estimate. The combined effect is a conservative preconditioning because a larger effective curvature leads to smaller steps in $\vm$, but possibly higher variance in the curvature estimator. When $\postdeformidx<1$ the $(2-\postdeformidx)$-escort concentrates on the interior of the support. Qualitatively, in this case, the mean is more sensitive to the bounded-support shape, while the scale update has lower variance, leading to a stable scale estimate, with bounded exploration due to the compact support.

We make additional adjustments to the generalized BLR $q$-VON updates in \cref{eq:studentt_sigma_nat_update,eq:studentt_m_update} and \cref{eq:bounded_s_update,eq:bounded_m_update}, in order to easily implement them as optimizers in modern deep-learning frameworks. Our simplifications are analogous to \citet[App.~D]{khan2018fast}'s and \citet{shen2402variational}'s modifications of the BLR that lead to deep-learning-friendly VON and IVON algorithms.
First, the loss function is reweighted with an inverse temperature $\zeta > 0$ and scaled by $1/N$, that is, we use $\zeta \frac{1}{N} \sum_{i=1}^N \ell_i(\vparam)$ in \cref{eq:fy_obj}. This reweighting allows us to use mini-batch stochastic gradients along with sampling to estimate the expectations (lines~2 of \cref{alg:qvon} and \cref{alg:tvon}). 
Second, we use a diagonal $\vS = \text{diag}(\vs)$ and reparametrize it via $\vs = \zeta (\vh + \delta)$ (lines~7), where $\delta > 0$ is the $L_2$-regularizer, $-\eta_0(\vparam) = \frac{\zeta \delta}{2} \| \vparam \|^2$ in the generalized Bayesian objective of \cref{eq:fy_obj}. The $\vS$-updates in \cref{eq:bounded_s_update,eq:studentt_sigma_nat_update} are rewritten in terms of $\vh$ in lines~5 of the algorithms, and separate learning rates $(1 - \beta_2)$ and $\alpha_t$ are introduced in place of $\rho_t$. Similar to IVON~\citep{shen2402variational}, we add a gradient momentum in lines~4. This is not a direct consequence of the theory, but does not change the fixed-points of the methods and we found it improves results. The initial $\vm$ is set to common initialization schemes, and $\vS = \zeta (h_0 + \delta)\vI$ where $h_0 > 0$ is a hyperparameter. To stabilize the algorithms, we ensure that the noise added to the parameters remains bounded by clipping the entries of $\vs$ to a minimum $s_\text{min} > 0$.


\subsection{Mean-Field Variants}
\label{ssec:meanfield}
The constants  $\kappa_t$, $\kappa$ and $\gamma$ used in \cref{alg:qvon} and \cref{alg:tvon} depend on the dimensionality of the parameter $\vparam$. To further adapt the proposed $q$-VON algorithms to high-dimensional problems, we also propose mean-field variants, where each parameter follows a one-dimensional $q$-Gaussian. This corresponds to a factorized posterior, that is, $\pi(\vparam) = \pi(\param_1) \cdot \hdots \cdot \pi(\param_D)$. Mean-field variants of \cref{alg:qvon} and \cref{alg:tvon} are easily obtained when $\Omega$ in the generalized Bayesian variational problem~\cref{eq:fy_obj} is replaced with the sum of one-dimensional functions on each parameter. 

In practice, the only modification needed to \Cref{alg:qvon} and \Cref{alg:tvon} is to set $D=1$ in the computation of the constants, and compute $\vsigma^2$ in lines~8 individually for each $s_i$ rather than taking their product. The sampling step in lines~2 is also modified to sample from a one-dimensional $q$-Gaussian or Student-t distribution. In the remainder of the paper, we refer to the mean-field $q$-VON variants as $q$-VON-MF.

\section{Experiments}\label{app:experiments}

\subsection{Bayesian logistic regression}

We demonstrate how $q$-VON-MF with bounded support  (\Cref{alg:qvon},  \Cref{sec:qvon_summary}) can be used to approximate the posterior mode in Bayesian logistic regression on the UCI Credit Approval dataset (\Cref{fig:fig3}). Posterior marginals for each feature are computed with MCMC. We compare variational posterior approximations: mean-field Gaussian ($q=1$) to bounded-support $q$-Gaussians with $q=-2$ and $q=-10$.

\subsection{Image Classification with Neural Networks}
We use $q$-VON (\Cref{alg:qvon} and \Cref{alg:tvon}) and the mean-field variant $q$-VON-MF (described in \Cref{sec:qvon_summary}) to train neural networks on the MNIST, CIFAR-10 and CIFAR-100 datasets. We compare with several baseline learning algorithms on test accuracy, test loss (NLL), expected calibration error~\cite{guo2017calibration}, Brier score and the AUROC metric introduced by~\citet{osawa2019practical}. The baselines are SGD (with momentum) applied to the standard empirical risk minimization objective to find the MAP, variational SGD (VSGD) which is SGD with Gaussian noise injection~\citep{orvieto2022anticorrelated}, and the Improved Variational Online Newton method (IVON) \citep{shen2402variational}. In all experiments, we closely follow the  setup of \citet{shen2402variational} and use their publicly available code. In particular, we use the same hyperparameters for $q$-VON and $q$-VON-MF as IVON. IVON and VSGD target the standard Bayesian objective for Gaussian families with fixed isotropic variance and flexible diagonal variance respectively. Another baseline we consider is $q$-VSGD, which targets the generalized Bayesian objective but only learns the mean of the approximate posterior distribution. All algorithms have the same complexity and run-time, as each iteration is dominated by the cost of forward and backward propagation through the neural network. The results are in \Cref{tab:deeplearning} for MNIST and \Cref{tab:deeplearning2} for CIFAR-10 and CIFAR-100. Below we summarize our main findings.

\textit{$q$-VON consistently improves over $q$-VSGD and SGD.} Both on MNIST and CIFAR-10, we see that $q$-VON consistently improves over $q$-VSGD and SGD (MAP). This confirms that learning the shape of the approximate posterior rather than using a fixed shape improves both accuracy and calibration performance, and that the method is superior to pure noise injection approaches~\citep{orvieto2022anticorrelated}. 

\textit{$q$-VON can exceed IVON's performance. } On the MNIST dataset, the heavy-tailed variant of $q$-VON ($q > 1$) gives better generalization performance than IVON, whereas the compact support variant ($q<1$) gives better calibration. On CIFAR-10 and CIFAR-100, $q$-VON-MF reaches competitive performance to IVON, without tuning any of its hyperparameters. 

\textit{Posterior sampling helps on the CIFAR datasets.} For $q$-VON and IVON which estimate a distribution over weights, using the posterior predictive distribution generally boosts the performance. 

Overall, the experiments demonstrate that the proposed $q$-VON algorithms and the mean field variants easily scale to large neural networks. The methods significantly improve over MAP estimation and reach competitive performance with the recent state-of-the-art method IVON~\citep{shen2402variational}. We used the same setup and hyperparameters from~\citep{shen2402variational} (which were tuned for Gaussian posteriors). Using IVON-tuned hyperparameters, $q$-VON is already competitive; dedicated tuning remains future work.

\cref{tab:deeplearning2,tab:deeplearning} report results on deep learning with the proposed $q$-VON algorithms,  as described in \Cref{sec:qvon_summary}, App.~\ref{sec:qvon}.
$q$-VON ($q<1$) denotes \Cref{alg:qvon} with bounded-support Gaussian posterior and $q=0.6$. $q$-VON ($q>1$) denotes \Cref{alg:tvon} with Student-$t$ posterior with $\nu = 5 \cdot 10^5$. $q$-VON-MF refers to the mean-field variant; each parameter has its own $q$-Gaussian ($q=0.6$ for $q<1$ and $\nu=5$ for $q>1$).
 
\begin{table*}[h!]
\centering
\caption{
Results on convolutional neural networks. We evaluate various $q$-VON variants using the same hyperparameters used in~\citet{shen2402variational} for IVON. Even without tuning, the results are comparable to IVON, with the heavy-tailed $q$-VON-MF reaching slightly higher AUROC scores.}
\setlength{\tabcolsep}{4pt}
\renewcommand{\arraystretch}{1.12}
\resizebox{\textwidth}{!}{%
\begin{tabular}{llccccc}
\toprule
Task & Method & Acc. $(\uparrow)$ & NLL $(\downarrow)$ & ECE $(\downarrow)$ & Brier $(\downarrow)$ & AUROC $(\uparrow)$ \\
\midrule
\multirow{8}{*}{\parbox{2.2cm}{CIFAR-10\\ResNet-20\\270k params}} & SGD (MAP) & 91.2 $\pm$ 0.1 & 0.292 $\pm$ 0.001 & 3.7 $\pm$ 0.1 & 0.133 $\pm$ 0.001 & 91.9 $\pm$ 0.1 \\
 & VSGD & 91.9 $\pm$ 0.0 & 0.310 $\pm$ 0.005 & 4.5 $\pm$ 0.1 & 0.127 $\pm$ 0.001 & 92.1 $\pm$ 0.1 \\
 & IVON &  92.4 $\pm$ 0.1 & { 0.259} $\pm$ 0.006 & { 3.5} $\pm$ 0.2 & { 0.117} $\pm$ 0.002 & 92.4 $\pm$ 0.0 \\
& ~~+ sampling & {\bf 92.6} $\pm$ 0.1 & {\bf 0.221} $\pm$ 0.003 & {\bf 0.8} $\pm$ 0.1 & {\bf 0.109} $\pm$ 0.002 & 92.6 $\pm$ 0.1 \\
  & q-VSGD ($q<1$) & 91.9 $\pm$ 0.0 & 0.309 $\pm$ 0.008 & 4.5 $\pm$ 0.1 & 0.128 $\pm$ 0.002 & 92.0 $\pm$ 0.1 \\
 & q-VON ($q<1$) & 92.3 $\pm$ 0.2 & 0.271 $\pm$ 0.008 & 3.9 $\pm$ 0.2 & 0.119 $\pm$ 0.003 & 92.6 $\pm$ 0.2 \\
 & ~~+ sampling & 92.4 $\pm$ 0.2 & 0.243 $\pm$ 0.006 & 2.6 $\pm$ 0.2 & 0.113 $\pm$ 0.003 & {\bf 92.7} $\pm$ 0.2 \\
 & q-VON-MF ($q>1$) & 92.4 $\pm$ 0.0 & 0.270 $\pm$ 0.003 & 3.8 $\pm$ 0.1 & 0.118 $\pm$ 0.001 & 92.5 $\pm$ 0.2 \\
 & ~~+sampling  & 92.5 $\pm$ 0.0 & 0.229 $\pm$ 0.002 & 1.5 $\pm$ 0.1 & 0.110 $\pm$ 0.001 & {\bf 92.7} $\pm$ 0.1 \\
\midrule 
\multirow{8}{*}{\parbox{2.2cm}{CIFAR-100\\PreResNet-110\\4M params}} & SGD (MAP) & 73.0 $\pm$ 0.2 & 1.270 $\pm$ 0.009 & 14.1 $\pm$ 0.1 & 0.410 $\pm$ 0.002 & 86.1 $\pm$ 0.1 \\
 & VSGD & 73.8 $\pm$ 0.2 & 1.319 $\pm$ 0.009 & 15.5 $\pm$ 0.2 & 0.410 $\pm$ 0.003 & 86.1 $\pm$ 0.1 \\
 & IVON & {\ 75.1} $\pm$ 0.3 & { 1.150} $\pm$ 0.012 & { 13.8} $\pm$ 0.3 & { 0.382} $\pm$ 0.004 & { 86.8} $\pm$ 0.2 \\
 & ~~ + sampling & {\bf 75.7} $\pm$ 0.2 & {\bf 0.880} $\pm$ 0.006 & {\bf 2.9} $\pm$ 0.2 & {\bf 0.336} $\pm$ 0.003 & { 87.0} $\pm$ 0.2 \\
 & q-VSGD ($q<1$) & 73.3 $\pm$ 0.1 & 1.326 $\pm$ 0.001 & 15.6 $\pm$ 0.1 & 0.416 $\pm$ 0.002 & 86.1 $\pm$ 0.2 \\
 & q-VON ($q<1$) & 74.6 $\pm$ 0.2 & 1.216 $\pm$ 0.008 & 14.0 $\pm$ 0.1 & 0.391 $\pm$ 0.002 & 86.3 $\pm$ 0.1 \\
 & ~~+ sampling & 74.7 $\pm$ 0.2 & 1.189 $\pm$ 0.008 & 13.3 $\pm$ 0.1 & 0.387 $\pm$ 0.002 & 86.4 $\pm$ 0.1 \\
 & q-VON-MF ($q>1$) & 74.9 $\pm$ 0.1 & 1.204 $\pm$ 0.008 & 14.4 $\pm$ 0.1 & 0.389 $\pm$ 0.002 & 86.7 $\pm$ 0.1 \\
 & ~~ + sampling & 75.4 $\pm$ 0.1 & 0.930 $\pm$ 0.005 & 5.6 $\pm$ 0.1 & 0.344 $\pm$ 0.001 & {\bf 87.1} $\pm$ 0.1\\
\addlinespace[2pt]
\bottomrule
\end{tabular}%
}
\label{tab:deeplearning2}
\end{table*}

\begin{table*}[h!]
\centering
\caption{Additional results on MNIST. Heavy-tailed posteriors can improve generalization performance, whereas truncated ones can improve calibration. }
\setlength{\tabcolsep}{4pt}
\renewcommand{\arraystretch}{1.12}
\resizebox{\linewidth}{!}{%
\begin{tabular}{llccccc}
\toprule
Task & Method & Acc. $(\uparrow)$ & Test Loss $(\downarrow)$ & ECE $(\downarrow)$ & Brier $(\downarrow)$ & AUROC $(\uparrow)$ \\
\midrule
\addlinespace[2pt]
\midrule
\multirow{13}{*}{\parbox{2.2cm}{MNIST\\MLP \\ 550k params}} & SGD (MAP) & 98.69 $\pm$ 0.01 & 0.043 $\pm$ 0.000 & 0.37 $\pm$ 0.02 & 0.020 $\pm$ 0.000 & 97.87 $\pm$ 0.07 \\
 & VSGD & 98.75 $\pm$ 0.03 & 0.048 $\pm$ 0.000 & 0.60 $\pm$ 0.02 & 0.020 $\pm$ 0.000 & 97.68 $\pm$ 0.05 \\
 & IVON & 98.74 $\pm$ 0.02 & 0.048 $\pm$ 0.001 & 0.60 $\pm$ 0.02 & 0.020 $\pm$ 0.000 & 97.52 $\pm$ 0.10 \\
 & ~~+ sampling & 98.72 $\pm$ 0.01 & 0.050 $\pm$ 0.000 & 1.53 $\pm$ 0.02 & 0.022 $\pm$ 0.000 & 97.88 $\pm$ 0.05 \\
 \addlinespace[2pt]
\cmidrule(lr){2-7} \addlinespace[2pt] 
 \addlinespace[2pt]
 & q-VSGD ($q<1$) & 98.62 $\pm$ 0.02 & 0.045 $\pm$ 0.001 & 0.46 $\pm$ 0.01 & 0.021 $\pm$ 0.000 & 97.94 $\pm$ 0.06 \\
 & q-VON ($q<1$) & 98.69 $\pm$ 0.02 & {\bf 0.042} $\pm$ 0.000 & {\bf 0.40} $\pm$ 0.03 & 0.020 $\pm$ 0.000 & 98.04 $\pm$ 0.05 \\
 & ~~+ sampling & 98.69 $\pm$ 0.02 & {\bf 0.042} $\pm$ 0.000 & {\bf 0.40} $\pm$ 0.04 & 0.020 $\pm$ 0.000 & 98.04 $\pm$ 0.05 \\
 & q-VSGD ($q>1$) & 98.59 $\pm$ 0.01 & 0.045 $\pm$ 0.000 & 0.48 $\pm$ 0.02 & 0.022 $\pm$ 0.000 & 97.97 $\pm$ 0.05 \\
 & q-VON ($q>1$) & {\bf 98.81} $\pm$ 0.03 & 0.043 $\pm$ 0.000 & 0.60 $\pm$ 0.01 & {\bf 0.019} $\pm$ 0.000 & 97.84 $\pm$ 0.11 \\
 & ~~+ sampling & 98.76 $\pm$ 0.02 & {\bf 0.042} $\pm$ 0.000 & 0.85 $\pm$ 0.01 & 0.020 $\pm$ 0.000 & {\bf 98.05} $\pm$ 0.09 \\
 \addlinespace[2pt]
\cmidrule(lr){2-7} 
\addlinespace[2pt] 
& q-VON-MF ($q<1$) & 98.66 $\pm$ 0.03 & 0.052 $\pm$ 0.001 & 0.60 $\pm$ 0.02 & 0.021 $\pm$ 0.000 & 97.19 $\pm$ 0.16 \\
& ~~+ sampling & 98.62 $\pm$ 0.02 & 0.053 $\pm$ 0.001 & 1.60 $\pm$ 0.02 & 0.024 $\pm$ 0.000 & 97.82 $\pm$ 0.04 \\
 & q-VON-MF ($q>1$) & 98.76 $\pm$ 0.02 & 0.049 $\pm$ 0.001 & 0.59 $\pm$ 0.02 & 0.020 $\pm$ 0.000 & 97.14 $\pm$ 0.07 \\
 & ~~+sampling & 98.67 $\pm$ 0.01 & 0.054 $\pm$ 0.001 & 1.90 $\pm$ 0.02 & 0.023 $\pm$ 0.000 & 97.88 $\pm$ 0.05 \\
\addlinespace[2pt]
\bottomrule
\end{tabular}%
}
\label{tab:deeplearning}
\end{table*}

\newpage

\subsection{Empirical Verification of Convergence}

A convergence analysis of BLR/VON-type methods in the generality considered here is still largely open. Existing results cover narrower settings: \citet{wu2024understanding} prove non-asymptotic rates for stochastic natural gradient (NG) VI in conjugate exponential family models and show that for non-conjugate log-concave likelihoods, the negative ELBO need not be convex in the expectation-parameter coordinates used by canonical NGVI. 
\citet{guilmeau2025regularized} instead study Bregman-proximal algorithms for regularized R\'enyi-divergence minimization over exponential families, with convergence guarantees based on relative smoothness with respect to the log-partition geometry. These assumptions are different from ours: our geometry is induced by a density-space Fenchel-Young discrepancy, which defines the generalized maxent family and the corresponding BLR update. 
A convergence theorem for the Fenchel-Young BLR would be a substantial project in its own right and is outside the scope of this paper. 
Instead, we provide empirical evidence that the implemented generalized updates move toward lower values of their corresponding variational objectives.

We use 2D Bayesian logistic regression as a minimal nontrivial diagnostic, with bounded-support, Gaussian, and Student-$t$ variants. The plots in \cref{fig:convergence} show slices of the optimizer-specific generalized objective over one mean-dispersion coordinate, with the remaining coordinate fixed at its final value. 
Each panel recomputes the objective associated with that posterior family and parametrization, because the Fenchel-Young discrepancy, sampling law, and scale-precision map change with \(q\) and with the full versus mean-field construction. All trajectories move toward low-objective regions and stabilize near the minimum. 

\begin{figure*}[h!] 
    \begin{center}
    \includegraphics[width=\textwidth]{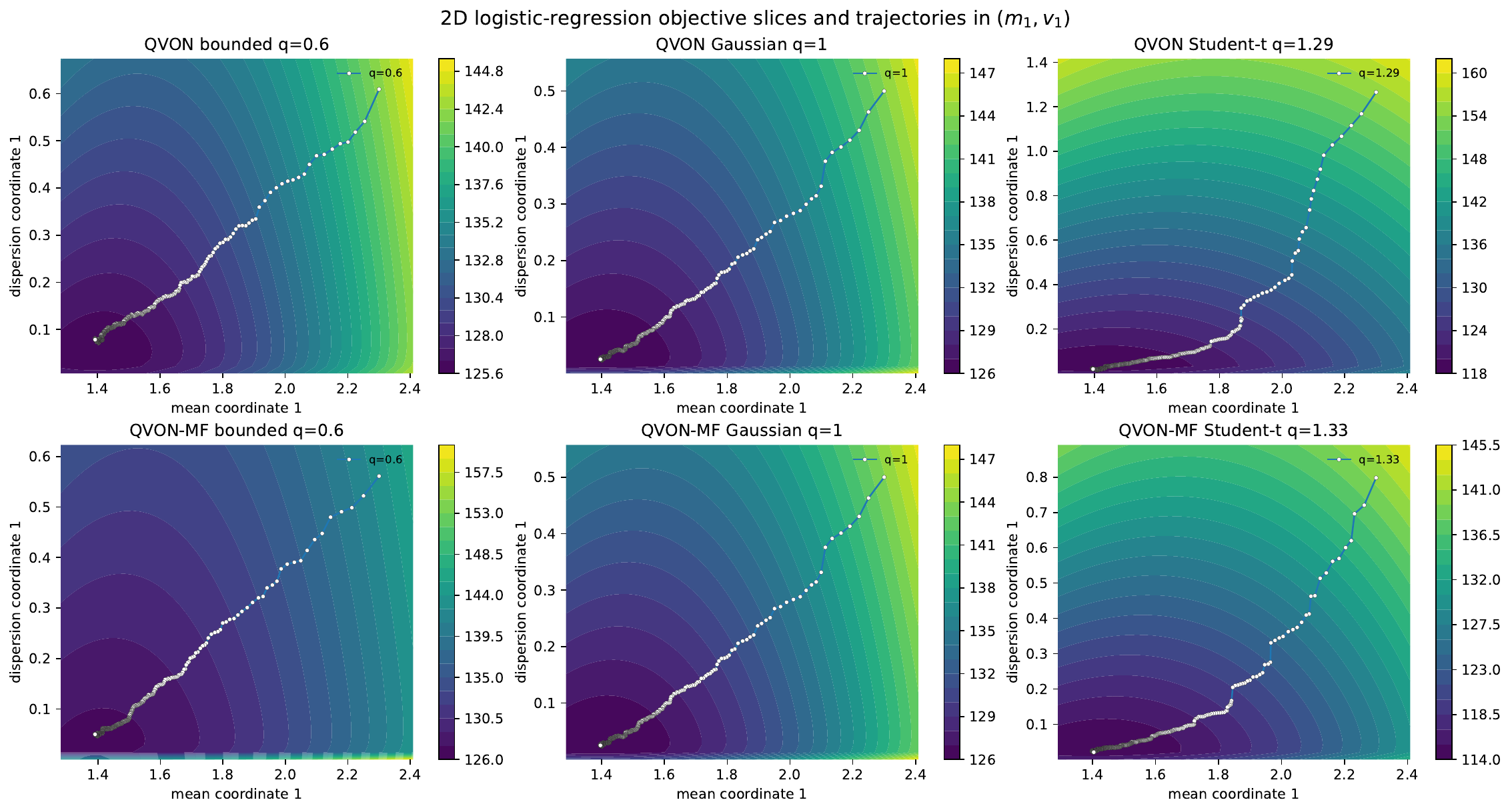}
    \caption{Optimizer-specific convergence on 2D Bayesian logistic regression. Each panel shows the objective associated with an optimizer and posterior family. The plots show slices of the optimizer-specific generalized objective over one mean-dispersion coordinate, with the remaining coordinate fixed at its final value. The trajectories consistently approach low-objective regions for bounded-support, Gaussian, and Student-$t$ variants, for both full and mean-field implementations. The same $\nu=5$ in both Student-$t$ panels corresponds to different $q$'s because of the dependence on $D$.}
    \label{fig:convergence}
\end{center}
\end{figure*}

\newpage 
\section{Varying $q$}\label{app:varying_q}

Different choices of \(q\) change the variational objective, the scale-precision map, and the stochastic Hessian estimator, so iteration counts across \(q\)'s are not directly comparable. Nevertheless, $q$-VON retains the fast
VON-style optimizer structure: a Newton-like mean update preconditioned by a learned diagonal precision, implemented with minibatch stochastic gradients, momentum, and Hessian-diagonal estimates rather than explicit second-order
computations \citep{khan2018fast, osawa2019practical, shen2402variational}.
Since $q$-VON changes only the sampling law, the precision constants, and the corresponding diagonal Stein estimator, we expect similar fast behavior.

\subsection{1D Logistic Regression}

In \cref{fig:1D_logreg} we use a synthetic dataset of size \(N=256\), with covariates \(x_i \sim \mathcal N(0,1)\), labels \(y_i \sim \mathrm{Bernoulli}(\sigma(2x_i))\), and Gaussian prior \(w \sim \mathcal N(0,0.1)\). For each \(q\), $q$-VON is run for 3000 steps and repeated over 5 seeds. Points show the average posterior mean or variance, and error bars show \(95\%\) confidence intervals. Horizontal reference lines show the exact Bayesian posterior quantities for the ordinary Gaussian-prior model. Exact references are approximated by evaluating the negative log-posterior \(\ell(w)+w^2/0.2\) on a grid.
\begin{figure*}[h!] 
    \begin{center}
    \includegraphics[width=\textwidth]{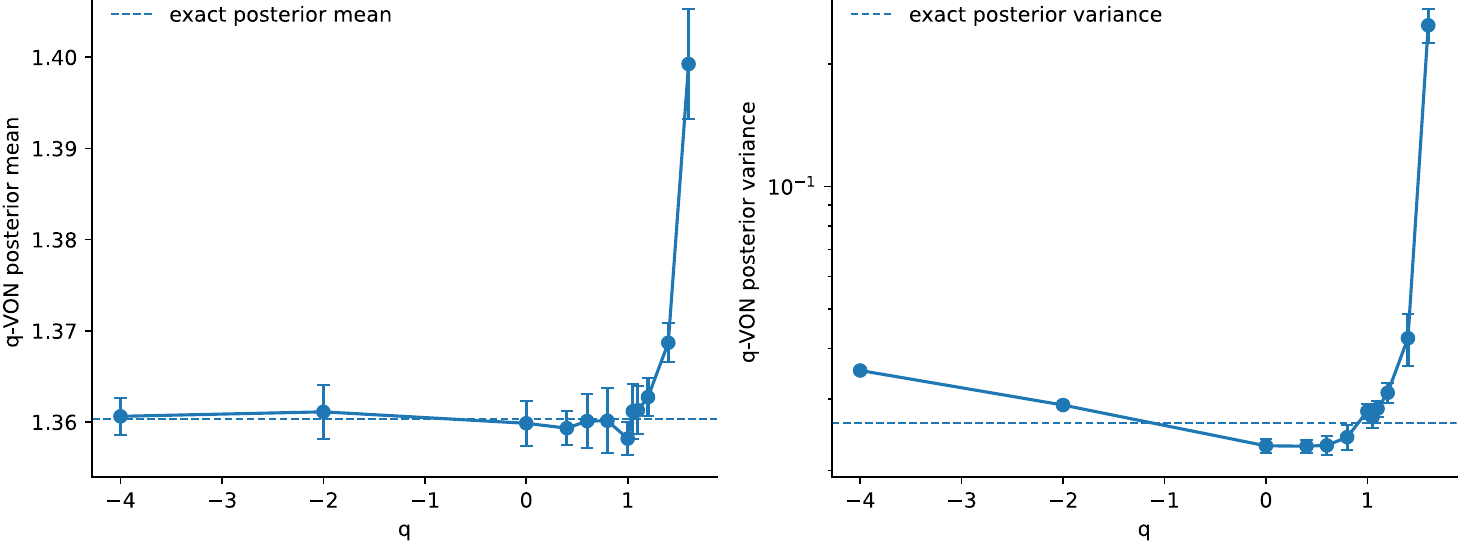}
\caption{
1D Bayesian logistic regression \(q\)-sweep with $q$-VON.
}\label{fig:1D_logreg}
\end{center}
\end{figure*}

\newpage

\subsection{MNIST}

In \cref{tab:mnist_qvon_mf_qsweep} we report the results of varying $q$ on MNIST training with $q$-VON-MF, again using the hyperparameters obtained from previous tuning with IVON.  
\begin{table*}[h!]
\centering
\caption{MNIST $q$-sweep for $q$-VON-MF. Results are mean $\pm$ s.e. over 5 seeds.}
\setlength{\tabcolsep}{4pt}
\renewcommand{\arraystretch}{1.12}
\resizebox{\linewidth}{!}{%
\begin{tabular}{llccccc}
\toprule
Task & Method & Acc. $(\uparrow)$ & Test Loss $(\downarrow)$ & ECE $(\downarrow)$ & Brier $(\downarrow)$ & AUROC $(\uparrow)$ \\
\midrule
\addlinespace[2pt]
\midrule
\multirow{25}{*}{\parbox{2.2cm}{MNIST\\MLP \\ 550k params}} & q-VON-MF ($q=-10$) & 81.42 $\pm$ 15.77 & 0.514 $\pm$ 0.412 & 1.35 $\pm$ 0.59 & 0.208 $\pm$ 0.165 & 88.55 $\pm$ 6.79 \\
 & q-VON-MF ($q=-8$) & 97.67 $\pm$ 0.03 & 0.086 $\pm$ 0.001 & 0.88 $\pm$ 0.01 & 0.036 $\pm$ 0.000 & 96.26 $\pm$ 0.18 \\
 & q-VON-MF ($q=-6$) & 98.15 $\pm$ 0.08 & 0.070 $\pm$ 0.002 & 0.86 $\pm$ 0.04 & 0.029 $\pm$ 0.001 & 96.79 $\pm$ 0.12 \\
 & q-VON-MF ($q=-4$) & 98.66 $\pm$ 0.03 & 0.053 $\pm$ 0.001 & 0.72 $\pm$ 0.01 & 0.022 $\pm$ 0.000 & 97.67 $\pm$ 0.12 \\
 & q-VON-MF ($q=-2$) & {\bf 98.75} $\pm$ 0.01 & 0.048 $\pm$ 0.001 & 0.61 $\pm$ 0.03 & {\bf 0.020} $\pm$ 0.000 & 97.52 $\pm$ 0.05 \\
 & q-VON-MF ($q=-1$) & 98.66 $\pm$ 0.02 & 0.047 $\pm$ 0.000 & 0.60 $\pm$ 0.02 & 0.021 $\pm$ 0.000 & 97.93 $\pm$ 0.07 \\
 & q-VON-MF ($q=0$) & 98.60 $\pm$ 0.01 & 0.046 $\pm$ 0.001 & 0.54 $\pm$ 0.02 & 0.021 $\pm$ 0.000 & 97.98 $\pm$ 0.05 \\
 & q-VON-MF ($q=0.4$) & 98.59 $\pm$ 0.03 & 0.047 $\pm$ 0.001 & 0.51 $\pm$ 0.02 & 0.022 $\pm$ 0.000 & 97.88 $\pm$ 0.05 \\
 & q-VON-MF ($q=0.6$) & 98.56 $\pm$ 0.02 & 0.047 $\pm$ 0.001 & 0.50 $\pm$ 0.02 & 0.022 $\pm$ 0.000 & 98.01 $\pm$ 0.03 \\
 & q-VON-MF ($q=0.8$) & 98.57 $\pm$ 0.02 & 0.046 $\pm$ 0.000 & 0.48 $\pm$ 0.02 & 0.022 $\pm$ 0.000 & 97.91 $\pm$ 0.05 \\
 & q-VON-MF ($q=0.95$) & 98.55 $\pm$ 0.04 & 0.047 $\pm$ 0.001 & 0.48 $\pm$ 0.03 & 0.022 $\pm$ 0.000 & 97.98 $\pm$ 0.06 \\
\addlinespace[2pt]
\cmidrule(lr){2-7}
\addlinespace[2pt]
 & q-VON-MF ($q=1$) & 98.57 $\pm$ 0.01 & 0.046 $\pm$ 0.001 & {\bf 0.45} $\pm$ 0.02 & 0.022 $\pm$ 0.000 & 98.02 $\pm$ 0.04 \\
\addlinespace[2pt]
\cmidrule(lr){2-7}
\addlinespace[2pt]
 & q-VON-MF ($q=1.000004$) & 98.53 $\pm$ 0.02 & 0.046 $\pm$ 0.000 & 0.50 $\pm$ 0.02 & 0.022 $\pm$ 0.000 & 98.03 $\pm$ 0.04 \\
 & q-VON-MF ($q=1.02$) & 98.55 $\pm$ 0.03 & 0.046 $\pm$ 0.001 & 0.48 $\pm$ 0.01 & 0.022 $\pm$ 0.000 & 97.95 $\pm$ 0.08 \\
 & q-VON-MF ($q=1.05$) & 98.53 $\pm$ 0.01 & 0.046 $\pm$ 0.000 & 0.49 $\pm$ 0.01 & 0.022 $\pm$ 0.000 & 98.05 $\pm$ 0.07 \\
 & q-VON-MF ($q=1.1$) & 98.56 $\pm$ 0.01 & 0.046 $\pm$ 0.000 & 0.50 $\pm$ 0.02 & 0.022 $\pm$ 0.000 & 98.03 $\pm$ 0.03 \\
 & q-VON-MF ($q=1.2$) & 98.53 $\pm$ 0.02 & 0.046 $\pm$ 0.000 & 0.48 $\pm$ 0.03 & 0.022 $\pm$ 0.000 & 98.03 $\pm$ 0.03 \\
 & q-VON-MF ($q=4/3$) & 98.51 $\pm$ 0.01 & 0.046 $\pm$ 0.000 & 0.47 $\pm$ 0.02 & 0.022 $\pm$ 0.000 & 98.05 $\pm$ 0.04 \\
 & q-VON-MF ($q=1.5$) & 98.54 $\pm$ 0.01 & 0.046 $\pm$ 0.001 & 0.48 $\pm$ 0.03 & 0.022 $\pm$ 0.000 & 98.00 $\pm$ 0.03 \\
 & q-VON-MF ($q=1.57143$) & 98.55 $\pm$ 0.02 & 0.046 $\pm$ 0.001 & 0.48 $\pm$ 0.01 & 0.022 $\pm$ 0.000 & 97.99 $\pm$ 0.04 \\
 & q-VON-MF ($q=1.6$) & 98.55 $\pm$ 0.03 & 0.046 $\pm$ 0.000 & 0.48 $\pm$ 0.03 & 0.022 $\pm$ 0.000 & 98.02 $\pm$ 0.07 \\
 & q-VON-MF ($q=1.625$) & 98.54 $\pm$ 0.04 & 0.047 $\pm$ 0.001 & 0.48 $\pm$ 0.04 & 0.022 $\pm$ 0.000 & 98.01 $\pm$ 0.06 \\
 & q-VON-MF ($q=1.64516$) & 98.55 $\pm$ 0.03 & {\bf 0.046} $\pm$ 0.000 & 0.48 $\pm$ 0.04 & 0.022 $\pm$ 0.000 & {\bf 98.07} $\pm$ 0.06 \\
 & q-VON-MF ($q=1.65574$) & 98.55 $\pm$ 0.02 & 0.046 $\pm$ 0.001 & 0.46 $\pm$ 0.03 & 0.022 $\pm$ 0.000 & 97.93 $\pm$ 0.05 \\
 & q-VON-MF ($q=1.66445$) & 98.57 $\pm$ 0.03 & 0.046 $\pm$ 0.000 & 0.47 $\pm$ 0.03 & 0.022 $\pm$ 0.000 & 97.99 $\pm$ 0.07 \\
\addlinespace[2pt]
\bottomrule
\end{tabular}%
}
\label{tab:mnist_qvon_mf_qsweep}
\end{table*}



\end{document}

%% file: defns.tex
\definecolor{tablecolor}{rgb}{0.8,0.8,0.8}

\newcommand{\barloss}{\bar{\ell}}

\newcommand{\loss}{\ell}

\newcommand{\vparam}{\vtheta}
\newcommand{\param}{\theta}

\newcommand{\vnatparam}{\vlambda}

\newcommand{\dkl}[2]{\mathbb{D}_{\text{KL}}(#1 \, \|\, #2)}

\newcommand\cut[1]{}

\newcommand{\tmu}{\widetilde{\vmu}}

\newcommand{\squishlist}{
   \begin{list}{$\bullet$}
    { \setlength{\itemsep}{0pt}      \setlength{\parsep}{3pt}
      \setlength{\topsep}{3pt}       \setlength{\partopsep}{0pt}
      \setlength{\leftmargin}{1.5em} \setlength{\labelwidth}{1em}
      \setlength{\labelsep}{0.5em} } }

\newcommand{\squishlisttwo}{
   \begin{list}{$\bullet$}
    { \setlength{\itemsep}{0pt}    \setlength{\parsep}{0pt}
      \setlength{\topsep}{0pt}     \setlength{\partopsep}{0pt}
      \setlength{\leftmargin}{2em} \setlength{\labelwidth}{1.5em}
      \setlength{\labelsep}{0.5em} } }

\newcommand{\squishend}{
    \end{list}  }

\newtheorem{thm}{Theorem}{}
\newtheorem{prop}{Proposition}{}
{}

\newcommand{\half}{\mbox{$\frac{1}{2}$}}
\newcommand{\real}{\mbox{$\mathbb{R}$}}

\newcommand{\myang}[1]{\langle#1\rangle}
\newcommand{\myexpect}{\mathbb{E}}

\newcommand{\gauss}{\mbox{${\cal N}$}}

\newcommand{\myvec}[1]{\mbox{$\mathbf{#1}$}}
\newcommand{\myvecsym}[1]{\mbox{$\boldsymbol{#1}$}}

\newcommand{\vzero}{\mbox{$\myvecsym{0}$}}

\newcommand{\vmu}{\mbox{$\myvecsym{\mu}$}}

\newcommand{\vlambda}{\mbox{$\myvecsym{\lambda}$}}

\newcommand{\vtheta}{\mbox{$\myvecsym{\theta}$}}

\newcommand{\vsigma}{\mbox{$\myvecsym{\sigma}$}}
\newcommand{\vSigma}{\mbox{$\myvecsym{\Sigma}$}}

\newcommand{\va}{\mbox{$\myvec{a}$}}
\newcommand{\vb}{\mbox{$\myvec{b}$}}

\newcommand{\vg}{\mbox{$\myvec{g}$}}
\newcommand{\vh}{\mbox{$\myvec{h}$}}

\newcommand{\vm}{\mbox{$\myvec{m}$}}

\newcommand{\vs}{\mbox{$\myvec{s}$}}

\newcommand{\vu}{\mbox{$\myvec{u}$}}

\newcommand{\vz}{\mbox{$\myvec{z}$}}

\newcommand{\vF}{\mbox{$\myvec{F}$}}

\newcommand{\vI}{\mbox{$\myvec{I}$}}

\newcommand{\vM}{\mbox{$\myvec{M}$}}

\newcommand{\vS}{\mbox{$\myvec{S}$}}
\newcommand{\vT}{\mbox{$\myvec{T}$}}

\newcommand{\Var}{\mbox{$\operatorname{Var}$}}



%% file: refs.bib
@string{icml = {International Conference on Machine Learning (ICML)}}

@string{nips = {Advances in Neural Information Processing Systems (NeurIPS)}}

@string{cpal = {Conference on Parsimony and Learning (CPAL)}}

@article{martins2022sparse,
  title={Sparse continuous distributions and {F}enchel-{Y}oung losses},
  author={Martins, Andr{\'e} FT and Treviso, Marcos and Farinhas, Ant{\'o}nio and Aguiar, Pedro MQ and Figueiredo, M{\'a}rio AT and Blondel, Mathieu and Niculae, Vlad},
  journal={Journal of Machine Learning Research},
  volume={23},
  number={257},
  pages={1--74},
  year={2022}
}

@inproceedings{khan2017conjugate,
  title={Conjugate-{C}omputation {V}ariational {I}nference: {C}onverting {V}ariational {I}nference in {N}on-conjugate {M}odels to {I}nferences in {C}onjugate {M}odels},
  author={Khan, Mohammad and Lin, Wu},
  booktitle={Artificial Intelligence and Statistics},
  pages={878--887},
  year={2017},
  organization={PMLR}
}

@article{khan2023bayesian,
  title={The {B}ayesian {L}earning {R}ule},
  author={Khan, Mohammad Emtiyaz and Rue, H{\aa}vard},
  journal={Journal of Machine Learning Research},
  volume={24},
  number={281},
  pages={1--46},
  year={2023}
}

@article{amari2011geometry,
  title={Geometry of $q$-{E}xponential {F}amily of {P}robability {D}istributions},
  author={Amari, Shun-Ichi and Ohara, Atsumi},
  journal={Entropy},
  volume={13},
  number={6},
  pages={1170--1185},
  year={2011},
  publisher={MDPI}
}

@article{matsuzoe2011geometry,
  title={Geometry for $q$-exponential families},
  author={Matsuzoe, Hiroshi and Ohara, Atsumi},
  journal={Recent progress in differential geometry and its related fields},
  pages={55--71},
  year={2011},
  publisher={World Scientific}
}

@inproceedings{khan2018fast,
  title={Fast and {S}calable {B}ayesian {D}eep {L}earning by {W}eight-{P}erturbation in {A}dam},
  author={Khan, Mohammad Emtiyaz and Nielsen, Didrik and Tangkaratt, Voot and Lin, Wu and Gal, Yarin and Srivastava, Akash},
  booktitle=icml,
  year={2018},
}

@article{landsman2008stein,
  title={Stein's Lemma for elliptical random vectors},
  author={Landsman, Zinoviy and Ne{\v{s}}lehov{\'a}, Johanna},
  journal={Journal of Multivariate Analysis},
  volume={99},
  number={5},
  pages={912--927},
  year={2008},
  publisher={Elsevier}
}

@article{landsman2015some,
  title={Some {S}tein-type {I}nequalities for {M}ultivariate {E}lliptical {D}istributions and {A}pplications},
  author={Landsman, Zinoviy and Vanduffel, Steven and Yao, Jing},
  journal={Statistics \& Probability Letters},
  volume={97},
  pages={54--62},
  year={2015},
  publisher={Elsevier}
}

@inproceedings{osawa2019practical,
  title={Practical deep learning with {B}ayesian principles},
  author={Osawa, Kazuki and Swaroop, Siddharth and Khan, Mohammad Emtiyaz and Jain, Anirudh and Eschenhagen, Runa and Turner, Richard E and Yokota, Rio},
  booktitle=nips,
  year={2019}
}

@inproceedings{shen2402variational,
  title={Variational {L}earning is {E}ffective for {L}arge {D}eep {N}etworks},
  author={Shen, Yuesong and Daheim, Nico and Cong, Bai and Nickl, Peter and Marconi, Gian Maria and Bazan, Clement and Yokota, Rio and Gurevych, Iryna and Cremers, Daniel and Khan, Mohammad Emtiyaz and M\"ollenhoff, Thomas},
  booktitle=icml, 
  year={2024}
}

@article{Zellner1988,
    author = {Arnold Zellner},
    title = {Optimal {I}nformation {P}rocessing and {B}ayes’s {T}heorem},
    journal = {The American Statistician},
    year = {1988},
    volume = {42},
    number = {4},
    pages = {278--280}
}

@article{Blei2017,
    author = {David M. Blei and Alp Kucukelbir and Jon D. McAuliffe},
    title = {Variational {I}nference: {A} {R}eview for {S}tatisticians},
    journal = {Journal of the American Statistical Association },
    year = {2017},
volume = {112},
number = {518},
pages = {859--877}
}

@inproceedings{blondel2019learning,
  title={Learning classifiers with {Fenchel-Young} losses: {G}eneralized entropies, margins, and algorithms},
  author={Blondel, Mathieu and Martins, Andre and Niculae, Vlad},
  booktitle={The 22nd International Conference on Artificial Intelligence and Statistics},
  pages={606--615},
  year={2019},
  organization={PMLR}
}

@book{amari2016information,
  title={Information {G}eometry and its {A}pplications},
  author={Amari, Shun-{I}chi},
  volume={194},
  year={2016},
  publisher={Springer}
}

@article{wainwright2008graphical,
  title={Graphical {M}odels, {E}xponential {F}amilies, and {V}ariational {I}nference},
  author={Wainwright, Martin J and Jordan, Michael I and others},
  journal={Foundations and Trends in Machine Learning},
  volume={1},
  number={1--2},
  pages={1--305},
  year={2008},
  publisher={Now Publishers, Inc.}
}

@article{bissiri2016general,
  title={A {G}eneral {F}ramework for {U}pdating {B}elief {D}istributions},
  author={Bissiri, Pier Giovanni and Holmes, Chris C and Walker, Stephen G},
  journal={Journal of the Royal Statistical Society: Series B (Statistical Methodology)},
  volume={78},
  number={5},
  pages={1103--1130},
  year={2016},
  publisher={Wiley Online Library}
}

@article{knoblauch2022optimization,
  title={An {O}ptimization-centric {V}iew on {B}ayes' {R}ule: {R}eviewing and {G}eneralizing {V}ariational {I}nference},
  author={Knoblauch, Jeremias and Jewson, Jack and Damoulas, Theodoros},
  journal={Journal of Machine Learning Research},
  volume={23},
  number={132},
  pages={1--109},
  year={2022}
}

@book{naudts2011generalised,
  title={Generalised {T}hermostatistics},
  author={Naudts, Jan},
  year={2011},
  publisher={Springer Science \& Business Media}
}

@techreport{minka2005divergence,
  title={Divergence {M}easures and {M}essage {P}assing},
  author={Minka, Tom and others},
  year={2005},
  institution={Microsoft Research},
  publisher={Technical report, Microsoft Research}
}

@article{li2016renyi,
  title={R{\'e}nyi {D}ivergence {V}ariational {I}nference},
  author={Li, Yingzhen and Turner, Richard E},
  journal={Advances in neural information processing systems},
  volume={29},
  year={2016}
}

@book{murphy2012machine,
  title={Machine {L}earning: {A} {P}robabilistic {P}erspective},
  author={Murphy, Kevin P},
  year={2012},
  publisher={MIT press}
}

@article{ivon,
  title={Variational {L}earning is {E}ffective for {L}arge {D}eep {N}etworks},
  author={Shen, Yuesong and Daheim, Nico and Cong, Bai and Nickl, Peter and Marconi, Gian Maria and Bazan, Clement and Yokota, Rio and Gurevych, Iryna and Cremers, Daniel and Khan, Mohammad Emtiyaz and Moellenhoff, Thomas},
  booktitle=icml, 
  year={2024}
}

@inproceedings{alquier2021non,
  title={{N}on-{E}xponentially {W}eighted {A}ggregation: {R}egret {B}ounds for {U}nbounded {L}oss {F}unctions},
  author={Alquier, Pierre},
  booktitle=icml, 
  year={2021},
}

@article{amari2010information,
  title={Information geometry of divergence functions},
  author={Amari, Shun-{I}chi and Cichocki, Andrzej},
  journal={Bulletin of the {P}olish academy of sciences. Technical {S}ciences},
  volume={58},
  number={1},
  pages={183--195},
  year={2010},
  publisher={Polska Akademia Nauk. Czasopisma i Monografie PAN}
}

@article{amari2012geometry,
  title={Geometry of deformed exponential families: {I}nvariant, dually-flat and conformal geometries},
  author={Amari, Shun-{I}chi and Ohara, Atsumi and Matsuzoe, Hiroshi},
  journal={Physica A: Statistical Mechanics and its Applications},
  volume={391},
  number={18},
  pages={4308--4319},
  year={2012},
  publisher={Elsevier}
}

@article{grunwald2020fast,
  title={Fast {R}ates for {G}eneral {U}nbounded {L}oss {F}unctions: {F}rom {ERM} to {G}eneralized {B}ayes},
  author={Gr{\"u}nwald, Peter D and Mehta, Nishant A},
  journal={Journal of Machine Learning Research},
  volume={21},
  number={56},
  pages={1--80},
  year={2020}
}

@article{amari2009alpha,
  title={$\alpha$-{D}ivergence {I}s {U}nique, {B}elonging to Both $f$-Divergence and {B}regman {D}ivergence {C}lasses},
  author={Amari, Shun-{I}chi},
  journal={IEEE Transactions on Information Theory},
  volume={55},
  number={11},
  pages={4925--4931},
  year={2009},
  publisher={IEEE}
}

@article{opper2009variational,
  title={The {V}ariational {G}aussian {A}pproximation {R}evisited},
  author={Opper, Manfred and Archambeau, C{\'e}dric},
  journal={Neural computation},
  volume={21},
  number={3},
  pages={786--792},
  year={2009},
  publisher={MIT Press}
}

@inproceedings{guo2017calibration,
  title={On calibration of modern neural networks},
  author={Guo, Chuan and Pleiss, Geoff and Sun, Yu and Weinberger, Kilian Q},
  booktitle=icml, 
  year={2017},
}

@inproceedings{orvieto2022anticorrelated,
  title={Anticorrelated {N}oise {I}njection for {I}mproved {G}eneralization},
  author={Orvieto, Antonio and Kersting, Hans and Proske, Frank and Bach, Francis and Lucchi, Aurelien},
  booktitle=icml, 
  year={2022},
}

@inproceedings{qsteincpal,
  title={A {S}tein {I}dentity for $q$-{G}aussians with {B}ounded {S}upport},
  author={Sklaviadis, Sophia and Möllenhoff, Thomas and Figueiredo, Mario A. T. and Martins, Andre F. T. and Khan, Mohammad Emtiyaz},
  booktitle=cpal, 
  year={2026},
}

@article{nielsen2022statistical,
  title={Statistical {D}ivergences between {D}ensities of {T}runcated {E}xponential {F}amilies with {N}ested {S}upports: Duo {B}regman and duo {J}ensen divergences},
  author={Nielsen, Frank},
  journal={Entropy},
  volume={24},
  number={3},
  pages={421},
  year={2022},
  publisher={MDPI}
}

@misc{fyvi,
  title = {Fenchel-Young {V}ariational {L}earning},
  author = {Sklaviadis, Sophia and Möllenhoff, Thomas and Martins, Andre and Figueiredo, Mario},
  year = {2025},
  note = {arXiv:2502.10295 [cs.LG]},
  doi = {10.48550/arXiv.2502.10295},
  url = {https://arxiv.org/abs/2502.10295}
}

@inproceedings{lukashchukquotient,
  title={The {Q}uotient {B}ayesian {L}earning {R}ule},
  author={Lukashchuk, Mykola and Tr{\'e}sor, Rapha{\"e}l and Nuijten, Wouter WL and Senoz, Ismail and de Vries, Bert},
  booktitle={The Thirty-ninth Annual Conference on Neural Information Processing Systems}
}

@article{naudts2002deformed,
  title={Deformed exponentials and logarithms in generalized thermostatistics},
  author={Naudts, Jan},
  journal={Physica A: Statistical Mechanics and its Applications},
  volume={316},
  number={1-4},
  pages={323--334},
  year={2002},
  publisher={Elsevier}
}

@inproceedings{hernandez2016black,
  title={Black-box $\alpha$-divergence {M}inimization},
  author={Hernandez-Lobato, Jose and Li, Yingzhen and Rowland, Mark and Bui, Thang and Hern{\'a}ndez-Lobato, Daniel and Turner, Richard},
  booktitle={International conference on machine learning},
  pages={1511--1520},
  year={2016},
  organization={PMLR}
}

@book{gibbs1902elementary,
  title={Elementary {P}rinciples in {S}tatistical {M}echanics: {D}eveloped with {E}special {R}eference to the {R}ational {F}oundations of {T}hermodynamics},
  author={Gibbs, Josiah Willard},
  year={1902},
  publisher={C. Scribner's sons}
}

@article{jaynes1957information,
  title={Information {T}heory and {S}tatistical {M}echanics},
  author={Jaynes, Edwin T},
  journal={Physical review},
  volume={106},
  number={4},
  pages={620},
  year={1957},
  publisher={APS}
}

@article{williams1980bayesian,
  title={Bayesian {C}onditionalisation and the {P}rinciple of {M}inimum {I}nformation},
  author={Williams, Peter M},
  journal={The British Journal for the Philosophy of Science},
  volume={31},
  number={2},
  pages={131--144},
  year={1980},
  publisher={Oxford University Press}
}

@article{catoni2007pac,
  title={{PAC}-{B}ayesian {S}upervised {C}lassification: {T}he {T}hermodynamics of {S}tatistical {L}earning},
  author={Catoni, Olivier},
  journal={arXiv preprint arXiv:0712.0248},
  year={2007}
}

@inproceedings{zhang1999theoretical,
  title={Theoretical {A}nalysis of a {C}lass of {R}andomized {R}egularization {M}ethods},
  author={Zhang, Tong},
  booktitle={Proceedings of the twelfth annual conference on Computational learning theory},
  pages={156--163},
  year={1999}
}

@article{hoffman2013stochastic,
  title={Stochastic {V}ariational {I}nference},
  author={Hoffman, Matthew D and Blei, David M and Wang, Chong and Paisley, John},
  journal={Journal of machine learning research},
  year={2013}
}

@book{lasota2013chaos,
  title={Chaos, {F}ractals, and {N}oise: {S}tochastic {A}spects of {D}ynamics},
  author={Lasota, Andrzej and Mackey, Michael C},
  volume={97},
  year={2013},
  publisher={Springer Science \& Business Media}
}

@article{zhu2014bayesian,
  title={Bayesian {I}nference with {P}osterior {R}egularization and {A}pplications to {I}nfinite {L}atent {SVM}s},
  author={Zhu, Jun and Chen, Ning and Xing, Eric P},
  journal={The Journal of Machine Learning Research},
  volume={15},
  number={1},
  pages={1799--1847},
  year={2014},
  publisher={JMLR. org}
}

@article{jewson2018principles,
  title={Principles of {B}ayesian {I}nference {U}sing {G}eneral {D}ivergence {C}riteria},
  author={Jewson, Jack and Smith, Jim Q and Holmes, Chris},
  journal={Entropy},
  volume={20},
  number={6},
  pages={442},
  year={2018},
  publisher={MDPI}
}

@article{ghosh2016robust,
  title={Robust {B}ayes estimation using the density power divergence},
  author={Ghosh, Abhik and Basu, Ayanendranath},
  journal={Annals of the Institute of Statistical Mathematics},
  volume={68},
  number={2},
  pages={413--437},
  year={2016},
  publisher={Springer}
}

@article{miller2019robust,
  title={Robust {B}ayesian inference via coarsening},
  author={Miller, Jeffrey W and Dunson, David B},
  journal={Journal of the American Statistical Association},
  year={2019},
  publisher={Taylor \& Francis}
}

@article{ishwar2005existence,
  title={On the existence and characterization of the maxent distribution under general moment inequality constraints},
  author={Ishwar, Prakash and Moulin, Pierre},
  journal={IEEE transactions on information theory},
  volume={51},
  number={9},
  pages={3322--3333},
  year={2005},
  publisher={IEEE}
}

@article{touchette2005legendre,
  title={Legendre-{F}enchel transforms in a nutshell},
  author={Touchette, Hugo},
  journal={URL http://www. maths. qmul. ac. uk/ht/archive/lfth2. pdf},
  pages={25},
  year={2005}
}

@book{boyd2004convex,
  title={Convex {O}ptimization},
  author={Boyd, Stephen and Vandenberghe, Lieven},
  year={2004},
  publisher={Cambridge University Press}
}

@article{wu2024understanding,
  title={Understanding {S}tochastic {N}atural {G}radient {V}ariational {I}nference},
  author={Wu, Kaiwen and Gardner, Jacob R},
  journal={arXiv preprint arXiv:2406.01870},
  year={2024}
}

@article{guilmeau2025regularized,
  author  = {Guilmeau, Thomas and Chouzenoux, Emilie and Elvira, V{\'i}ctor},
  title   = {Regularized R{\'e}nyi Divergence Minimization through
             Bregman Proximal Gradient Algorithms},
  journal = {Journal of Machine Learning Research},
  year    = {2025},
  volume  = {26},
  number  = {157},
  pages   = {1--56},
  url     = {https://www.jmlr.org/papers/v26/23-0573.html}
}
